\pdfoutput=1
\documentclass[twoside,11pt]{article}

\usepackage[preprint]{jmlr2e}

\usepackage{mathtools}
\usepackage{amssymb}
\usepackage{bm}
\usepackage{subcaption}
\usepackage{float}
\usepackage{booktabs}
\usepackage{xcolor}
\usepackage{tikz}
\usetikzlibrary{plotmarks}
\usetikzlibrary{arrows.meta}
\usepackage{lastpage}
\usepackage{placeins}
\usepackage[nameinlink,noabbrev]{cleveref}

\DeclareUnicodeCharacter{221E}{\ensuremath{\infty}}

\makeatletter
\let\proof\@undefined
\let\endproof\@undefined
\providecommand{\proofname}{Proof}
\newenvironment{proof}[1][\proofname]{\par
  \normalfont\topsep6\p@\@plus6\p@\relax
  \trivlist\item[\hskip\labelsep\itshape#1\@addpunct{.}]\ignorespaces
}{\hfill\BlackBox\endtrivlist}
\makeatother

\newsavebox{\rowbox}
\newlength{\rowht}

\newtheorem{assumption}{Assumption}

\newcommand{\R}{\mathbb{R}}
\newcommand{\N}{\mathbb{N}}
\newcommand{\PP}{\mathbb{P}}
\newcommand{\Mfd}{M}
\newcommand{\Lop}{\mathcal{L}}
\newcommand{\Aop}{\mathcal{A}}
\newcommand{\Bop}{\mathcal{B}}
\newcommand{\Pip}{\Pi_p}
\newcommand{\dvol}{\mathrm{dVol}}
\newcommand{\Rest}{R_N}
\newcommand{\khat}{\widehat{k}}
\newcommand{\Khat}{\widehat{K}}
\newcommand{\inj}{\iota}
\newcommand{\nullset}{\mathcal{N}}
\DeclareMathOperator{\dist}{dist}
\DeclareMathOperator{\spn}{span}
\DeclareMathOperator{\diagop}{diag}
\DeclareMathOperator{\grad}{grad}

\definecolor{cGFDM}{HTML}{E69F00}
\definecolor{cPMMex}{HTML}{000000}
\definecolor{cHex}{HTML}{56B4E9}
\definecolor{cHgr}{HTML}{D55E00}
\definecolor{cRhogr}{HTML}{CC79A7}

\DeclareRobustCommand{\lsample}[3]{%
  \tikz[baseline=-0.55ex]{%
    \draw[#1,line width=0.9pt,#2] (0,0) -- (0.62,0);
    \draw[#1,solid,mark=#3,mark size=1.9pt] plot coordinates {(0.31,0)};}%
}

\DeclareRobustCommand{\lsGFDM} {\lsample{cGFDM}{dash pattern=on 3pt off 2pt}{none}}
\DeclareRobustCommand{\lsPMMex}{\lsample{cPMMex}{dash pattern=on 0.8pt off 1.6pt}{none}}
\DeclareRobustCommand{\lsHex}  {\lsample{cHex}{solid}{triangle*}}
\DeclareRobustCommand{\lsHgr}  {\lsample{cHgr}{solid}{*}}
\DeclareRobustCommand{\lsRhogr}{\lsample{cRhogr}{solid}{square*}}
\newcommand{\pcell}[2]{\makebox[\dimexpr\linewidth/#1\relax][c]{\footnotesize #2}}

\jmlrheading{}{2026}{}{}{}{}{Federico Pavesi, No\'emie Jaquier and Antonio Candelieri}

\ShortHeadings{Probabilistic Meshless Methods on Manifolds}{Pavesi, Jaquier and Candelieri}
\firstpageno{1}

\begin{document}

\title{Probabilistic Meshless Methods on Manifolds}

\author{\name Federico Pavesi\thanks{Corresponding author.} \email f.pavesi11@campus.unimib.it \\
       \addr Department of Economics, Management and Statistics\\
       University of Milano-Bicocca\\
       Piazza dell'Ateneo Nuovo 1, 20126 Milan, Italy
       \AND
       \name No\'emie Jaquier \email jaquier@kth.se \\
       \addr Department of Robotics, Perception and Learning\\
       KTH Royal Institute of Technology\\
       Lindstedtsv\"agen 24, 10044 Stockholm, Sweden
       \AND
       \name Antonio Candelieri \email antonio.candelieri@unimib.it \\
       \addr Department of Economics, Management and Statistics\\
       University of Milano-Bicocca\\
       Piazza dell'Ateneo Nuovo 1, 20126 Milan, Italy
       }

\editor{My editor}

\maketitle

\begin{abstract}
Probabilistic meshless methods solve partial differential equations (PDEs) by placing a Gaussian process prior on the solution and conditioning it on the PDE at a finite set of collocation nodes. The resulting posterior provides both an approximation of the solution and a Bayesian quantification of its error. While these methods are well established on Euclidean domains, their extension to manifolds has received little attention and, to our knowledge, lacks any convergence analysis. In this paper, we extend probabilistic meshless methods to compact manifolds and study their convergence in two settings. First, for known manifolds whose Laplace--Beltrami eigenpairs are available, the prior can be evaluated exactly, but the differential operator of the PDE generally cannot be applied to it in closed form. We replace this operator by a standard generalized finite difference discretization built on a point cloud sampled from the manifold, and show that the resulting method is consistent at the points of the cloud, at a rate governed by their fill distance. Second, for implicit manifolds, which are only available through a point cloud, both the prior and the operator must be approximated. We propose a new generalized finite difference scheme in which the stencil radius is a free parameter. We show that this novel finite difference approximation of the operator retains the consistency of the standard construction while controlling the growth of its coefficients, and that an appropriate radius yields uniform convergence of the discrete Gram matrix. In both settings, we derive error bounds for the posterior mean which separate the collocation error intrinsic to probabilistic meshless methods from the error due to approximating the operator and, for implicit manifolds, the prior. Experiments on various manifolds such as the circle, the sphere, and the flat torus confirm our theoretical claims. Moreover, we validate our approach on real-world cortical thickness data, showing that it attains the accuracy of established discretizations, while conditioning on fewer than a tenth of the observations and providing a credible interval at every vertex.
\end{abstract}

\begin{keywords}
Probabilistic numerics, Gaussian processes, partial differential equations on manifolds, generalized finite differences, graph Laplacian
\end{keywords}

\begin{figure}[t]
\centering
\resizebox{\textwidth}{!}{%
\begin{tikzpicture}[
  x=1mm, y=1mm, font=\small,
  >={Stealth[length=2.2mm,width=1.8mm]},
  blk/.style={draw=black!55, rounded corners=2pt, align=center,
              inner sep=4pt, text width=36mm, minimum height=15mm},
  tag/.style={font=\scriptsize, align=center, inner sep=2pt}
]
\node[anchor=south] at (13,13) {$\Lop u=f$};
\begin{scope}[shift={(13,0)}]
  \draw[draw=black!45, fill=black!4]
    plot[smooth cycle, tension=0.85]
    coordinates {(-11,4) (-3,9) (7,5) (10,-4) (1,-9) (-9,-4)};
  \foreach \p in {(-7,2),(0,3),(7,0),(-2,-2),(1,-6),(5,-5),(-3,-6),(8,3),(-8,0)}
    \fill[black!65] \p circle (0.7);
  \foreach \p in {(-4,6),(4,5),(-6,-3),(3,-1)}
    \fill[cGFDM] \p circle (1.2);
\end{scope}
\node[tag, anchor=north] at (13,-13)
  {sampled points,\\ with $f$ evaluated\\ at the orange ones};
\node[blk, draw=cHex, fill=cHex!8] (ker) at (62,15)
  {a Gaussian prior\\ over the solution};
\node[blk, draw=cRhogr, fill=cRhogr!8] (op) at (62,-15)
  {a point cloud\\ approximation\\ of the operator $\Lop$};
\node[blk, fill=black!5] (sys) at (110,0)
  {condition the prior\\ on the values of $f$};
\begin{scope}[shift={(152,0)}]
  \draw[draw=black!45, fill=black!4]
    plot[smooth cycle, tension=0.85]
    coordinates {(-11,4) (-3,9) (7,5) (10,-4) (1,-9) (-9,-4)};
  \foreach \p in {(-7,2),(0,3),(7,0),(-2,-2),(1,-6),(5,-5),(-3,-6),(8,3),(-8,0),
                  (-4,6),(4,5),(-6,-3),(3,-1)}
    \fill[cHex] \p circle (0.7);
  \foreach \p in {(-7,2),(0,3),(7,0),(1,-6),(8,3),(-4,6),(3,-1)}
    \draw[cHex, line width=0.35pt] \p ++(0,1.6) -- ++(0,-3.2);
\end{scope}
\node[tag, anchor=north] at (152,-13) {$u$ at every point,\\ with its uncertainty};
\draw[->] (26,3)   -- (42,13);
\draw[->] (26,-3)  -- (42,-13);
\draw[->, rounded corners=2.5pt] (ker.east) -- (88,15) -- (88,3) -- ([yshift=3mm]sys.west);
\draw[->, rounded corners=2.5pt] (op.east)  -- (88,-15) -- (88,-3) -- ([yshift=-3mm]sys.west);
\draw[->] (sys.east) -- (138,0);
\end{tikzpicture}
}
\caption{Overview of the method. The equation is posed on a manifold known only through a sample of points, from which we build both a Gaussian prior over the solution and a point cloud approximation of the differential operator. Conditioning on the values of the forcing term at the collocation points (orange) returns a posterior, which gives the solution at every node together with its uncertainty.}
\label{fig:method}
\end{figure}
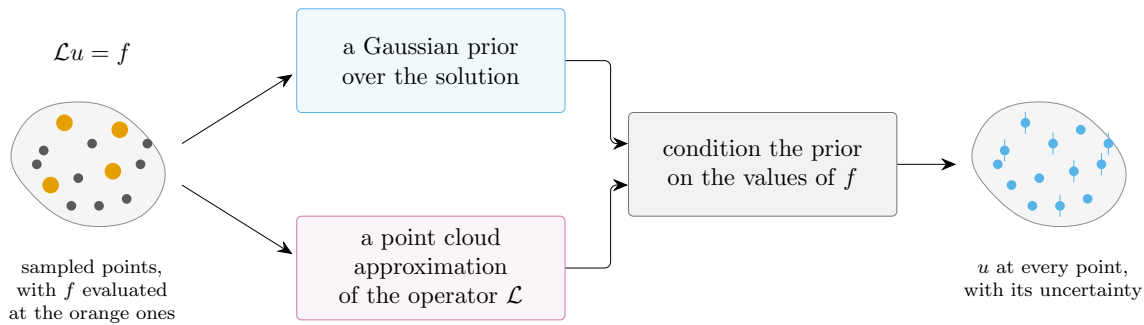
\section{Introduction}
Partial differential equations (PDEs) are among the most widely used tools in mathematics and physics as they provide the natural language for modelling a wide variety of natural phenomena~\citep{hilbert1985methods,evans2022partial}.
Given a domain $X$, a partial differential equation has the form
\[
 \Aop u = g
\]
where $g$ and $u$ are the forcing term and the solution, both belonging to function spaces over $X$, and $\Aop$ is a differential operator from the function space of $u$ to the one of $g$. When $X$ has a boundary, an additional condition of the form $\Bop u =b$ is typically imposed, where $\Bop$ and $b$ are defined as above but only on the boundary $\partial X$. For simplicity, in the case with boundary, we refer to the stacked terms $g=[g, b]^{\top}$ and operators $\Lop = [\Aop, \Bop]^{\top}$ as forcing term and differential operator, respectively.
In this work, we focus on linear differential operators $\Aop$ and $\Bop$. While the linear case is the simplest, it is amenable to model a wide range of phenomena such as heat diffusion~\citep{carslaw1962conduction}, electromagnetism~\citep{jackson1962classical}, and linear elasticity~\citep{love1944mathematical}.

In most cases of interest, PDEs do not admit closed-form solutions and must be solved numerically~\citep{quarteroni1994numerical, leveque2007finite, brenner2008mathematical}. All numerical solvers rely on a discretization, which replaces the infinite-dimensional problem by a finite-dimensional one. Among the most widely used approaches are finite elements~\citep{brenner2008mathematical}, finite volume~\citep{eymard2000finite}, generalized finite differences~\citep{mirzaei2012generalized}, and spectral methods~\citep{gottlieb1977}.
Finite elements and finite volume methods approximate the domain with a mesh. In finite element methods, the mesh induces a piecewise polynomial basis, and projecting the equation onto it turns the problem into a sparse linear system for the coefficients of the solution. This approach has proven to be extremely powerful, at the cost of the complexity of generating a mesh and of choosing a basis adapted to the geometry~\citep{quarteroni1994numerical}.
Generalized finite difference methods avoid the mesh and approximate the operator directly on a point cloud, with a matrix that reproduces the operator action on polynomials exactly up to a desired degree. The solution is then obtained by solving the resulting linear system, i.e., by inverting the discretized operator. However, the condition number of such matrix degrades with growing number of points and polynomial order, leading to asymptotically divergent behaviors~\citep{zheng2022theoretical}.
Spectral methods shift the discretization from the geometric domain to the function space. The eigenfunctions of the operator are taken as basis for both the solution and the forcing term. Expanding both in this basis, the operator multiplies each coefficient of the solution by the corresponding eigenvalue, so that matching the two expansions gives each coefficient of the solution as the coefficient of the forcing term divided by its eigenvalue. However, this requires the eigenpairs of the operator in closed form, which are available only in very specific cases~\citep{gottlieb1977}.
The output of all these methods is a numerical approximation of the PDE solution $u$.

Beyond the approximation itself, it is often desirable to quantify the uncertainty about its distance from the true solution. Gaussian processes provide a natural route to obtain both the approximation and an uncertainty estimate. A Gaussian process is a statistical distribution over functions, fully characterized by its mean function $m:X \to \R$ and covariance kernel $k:X \times X \to \R$~\citep{bogachev1998gaussian}, the latter determining most of its properties. A key property of Gaussian processes is that they are closed under conditioning. Namely, given a set of points $\{x_1, \dots , x_N\} =Y_N \subset X$ and their evaluations $u(Y_N)$ of the function to approximate, the posterior distribution is again a Gaussian process with moments
\[
\begin{aligned}
    &\mu(x|Y_N)= \bigl[k(x,Y_N)\bigr]\,
              \bigl[k(Y_N,Y_N)\bigr]^{-1}u(Y_N),\\
    &\Sigma(x,x'|Y_N) = k(x,x')-\bigl[k(x,Y_N)\bigr]\,
              \bigl[k(Y_N,Y_N)\bigr]^{-1}
              \bigl[k(Y_N,x')\bigr],
\end{aligned}
\]
where $k(Y_N,Y_N)$ is called the Gram matrix.
In the case where we aim at solving PDEs, the first step is to replace the solution by a Gaussian process which takes values in an appropriate function space. In the case of a PDE, however, the solution cannot be evaluated, and only the forcing term is accessible. The solution is therefore modeled as a Gaussian process $\hat u \sim \mathcal{GP}(m,k)$ and the prior model for the forcing term is obtained by applying the operator to it, $\hat g = \Lop \hat u$. Since $\Lop$ is linear, by closedness of Gaussians under linear operators, the forcing term $g$ is again Gaussian. Conditioning on evaluations of the forcing term at a finite set of collocation nodes then yields a posterior process whose mean approximates the solution and whose variance quantifies the uncertainty about it~\citep{sarkka2011, cialenco2012approximation, owhadi2015bayesian, cockayne2017probabilistic}. This approach is known as the \emph{probabilistic meshless method} (PMM)~\citep{cockayne2017probabilistic}, described in detail in Section~\ref{sec:PMMs_intro}.

PMMs have been extensively studied when the domain is an appropriate subset of a Euclidean space~\citep{sarkka2011, cialenco2012approximation, owhadi2015bayesian, cockayne2017probabilistic}. However, many applications involve PDEs defined on manifolds, such as granular flow~\citep{Suchde2019}, surface Navier-Stokes equations~\citep{nitschke2012finite}, liquid crystals~\citep{nestler2018orientational}, biomembranes~\citep{barreira2011surface, elliott2012modelling}, and computer graphics~\citep{bertalmio2001variational, crane2013geodesics}. Despite their relevance, manifold domains have received little attention in the PMM literature.

A notable exception is the work of~\citet{Harlim_2022}, which focuses on point clouds sampled from the domain being an unknown manifold with boundaries. In this setting, the Gaussian process on the manifold is approximated at the points of the cloud through a graph spectral approximation, as proposed by~\citet{Harlim2020}. Since the operator is not available in closed form, ~\citet{Harlim_2022} approximate it at the points by a matrix obtained via the ghost points diffusion method~\citep{jiang2023ghost}, which is constructed similarly to generalized finite differences. The restriction of a point cloud makes the Gaussian process a finite Gaussian vector and, since the operator is a matrix, the prior model for the forcing term $g$ is the Gaussian vector $u$ after ghost point matrix multiplication $L$, $L\,u$.
This work differs from ours in two respects.
First, it focuses on inverse problems, in which the operator is assumed to depend on an unknown parameter $\theta$ while observations of the solution are available, and whose objective is to infer $\theta$ rather than the solution itself.
Second, its convergence is assessed only through numerical experiments, without a theoretical analysis.

Inspired by this work, we extend PMMs to manifolds and provide thorough analysis of convergence rates in two settings, which differ in the prior information available about the manifold. First, we consider a \emph{known} manifold, whose Laplace--Beltrami eigenpairs are available. In this case, the base Gaussian process can be written as a Fourier series in these eigenpairs. The differential operator is then either applied to this series in closed form or approximated with a generalized finite difference method (GFDM). In the former case, the construction coincides with the Euclidean one, up to the notation. In the latter case, the Gaussian process over the solution is restricted to the point cloud on which the operator is discretized, and the prior model is obtained by GFDM matrix multiplication of the solution Gaussian process.
We prove that this model is consistent with the closed-form one at the points of the cloud, at a rate governed by the fill distance of the cloud.
Second, we consider an implicit manifold, which is only available through a point cloud sampled from it. In this case, we treat the problem similarly to~\citet{Harlim_2022}, since both the operator and the Gaussian process are not available in closed form. Restricting to closed manifolds, we replace the ghost point diffusion matrix with a novel GFDM scheme allowing to control the rate at which the norm of the GFDM matrix grows, thus resulting in additional stability over the action (left multiplication) of this matrix. We prove that the resulting Gram matrix converges to the closed form one at the points of the cloud.
Finally, for all models, we derive error bounds for the posterior mean, which separate the intrinsic collocation error from the error due to the approximation of the operator and of the Gaussian process.

The paper is organized as follows. Section~\ref{sec:background} reviews probabilistic meshless methods and generalized finite differences on manifolds. Section~\ref{sec:pmm} presents proposed extension to known and implicit manifolds, together with the corresponding convergence results. Section~\ref{sec:contraction} derives error bounds for the posterior mean at the nodes and the contraction speed to the true solution as the density of the point cloud grows and the number evaluations of the forcing term increases. Finally, Section~\ref{sec:experiments} reports numerical experiments on the circle, the sphere and the flat torus, where the exact solution is available and the theoretical results can be tested, as well as an application to the smoothing of cortical thickness data.

\section{Background}\label{sec:background}
This section first introduces the probabilistic meshless methods, the class of stochastic collocation solvers for PDEs on which our approach builds. Then, it reviews generalized finite difference methods for approximating differential operators on manifolds. These methods provide the discretization we rely on whenever the differential operator cannot be applied to the kernel in closed form, both on known and on implicit manifolds.

\subsection{Probabilistic Meshless Methods}
\label{sec:PMMs_intro}

Probabilistic meshless methods~\citep{cockayne2017probabilistic} are PDE solvers which belong to the broader class of Bayesian numerical methods~\citep{cockayne2019}.
These methods recast classical numerical tasks as problems of statistical inference, and have been developed for optimization~\citep{mockus1989application, mahsereci2017probabilistic}, integration~\citep{ohagan1991bayes}, linear systems~\citep{hennig2015probabilistic}, and ordinary and partial differential equations~\citep{sarkka2011, cialenco2012approximation, owhadi2015bayesian, cockayne2017probabilistic,tronarp2021bayesian}. The main advantage of Bayesian methods over classical methods is that the Bayesian posterior distribution provides both an estimate of the problem solution and a quantification of its uncertainty.
For PDEs, closely-related approaches have been developed under different formalisms, e.g., under the perspectives of stochastic partial differential equations~\citep{cialenco2012approximation}, Bayesian homogenization~\citep{owhadi2015bayesian} and partial differential equations~\citep{sarkka2011, cockayne2017probabilistic}.
In the following, we introduce PMMs using the formalism of~\citet{cockayne2017probabilistic}.

Consider a domain $X\subseteq\R^{d}$, possibly with Lipschitz boundary $\partial X$, where $d<\infty$. We write $\overline{X}:=X\cup\partial X$ for the union of the domain and its boundary.
The prototypical PDE system is defined as
\begin{equation}\label{eq:pdesystem}
\left\{
\begin{aligned}
    &\Aop u_0(x) = g(x), &&x\in X \\
    &\Bop u_0(x) = b(x), &&x\in\partial X
\end{aligned}
\right. ,
\end{equation}
where $\Aop:H(X)\to H_{\Aop}(X)$ and $\Bop:H(X) \to H_{\Bop}(\partial X)$ are linear operators, $g \in H_{\Aop}(X)$ and $b \in H_{\Bop}(\partial X)$ are the forcing terms, $H,H_{\Aop}$ and $H_{\Bop}$ are Hilbert spaces of functions where the sampling operator is well defined, and $u_0$ is the solution.
Following~\citet{cockayne2017probabilistic}, we allow for two sources of uncertainty in~\eqref{eq:pdesystem}. The first one arises in Bayesian inverse problems, which aim to estimate an unknown parameter $\theta \in \Theta$ of the operator from observations of the solution given the forcing term. In Bayesian terms, the unknown $\theta$ is modeled by a probability distribution over $\Theta$, which is updated to a posterior distribution given the observations. Formally, the uncertainty over the operator is modeled by introducing a probability space $(\Omega',\mathcal F',\PP')$ and setting $\theta=\theta(\omega')$, so that the operators are rewritten as
\[
   \Aop:H(X)\times\Omega'\to H_{\Aop}(X),
   \qquad
   \Bop:H(X)\times\Omega'\to H_{\Bop}(\partial X),
\]
both assumed to be Borel measurable operators. The solution becomes a random function $u(\cdot,\omega')\in H(X)$, defined by
\begin{equation}\label{eq:problem}
\left\{
\begin{aligned}
    \Aop[\omega']\,u(x,\omega') &= g(x), &&x\in X, \\
    \Bop[\omega']\,u(x,\omega') &= b(x), &&x\in\partial X,
\end{aligned}
\right.
\end{equation}
with $g\in H_{\Aop}(X)$ and $b\in H_{\Bop}(\partial X)$.
The second source of uncertainty arises in the forward problem, i.e., in the approximation of the solution of~\eqref{eq:pdesystem}. As advocated by~\citet{owhadi2015bayesian}, a natural construction places a Gaussian process prior over the forcing term, which is then propagated to the solution through the operator.
To this end, we introduce a second probability space $(\Omega,\mathcal F,\PP)$ and model the forcing terms as stochastic processes, such that $g:X\times\Omega\to\R$ satisfies $g(\cdot,\omega)\in H_{\Aop}(X)$ for almost every $\omega$, and similarly for $b$. This leads to the  doubly stochastic problem
\begin{equation}\label{eq:problem_relaxed}
    \left\{
\begin{aligned}
    \Aop[\omega']\,u(x,\omega', \omega) &= g(x, \omega), &&x\in X \\
    \Bop[\omega']\,u(x,\omega', \omega) &= b(x, \omega), &&x\in\partial X
\end{aligned}
\right. .
\end{equation}
This system encompasses the class of problems that the probabilistic meshless method can solve. Using two separate probability spaces allows a single model to cover several settings, such as inverse problems with noisy observations, or forward problems with an uncertain operator. The latter is the setting of this paper, where the uncertainty on the operator originates from its discretization.

Since Gaussian processes remain Gaussian under linear operators~\citep{bogachev1998gaussian}, it is convenient to model the forcing term as a Gaussian process on $H_{\Lop}(X)$ with covariance kernel $\Lambda$, and to derive the distribution of the solution $u$ from it.
To this end, we fix $\omega' \in \Omega'$ and write $\Lop=\Lop[\omega']$. The Green's function is defined as the solution of the system
\begin{equation}
    \Lop\,G(x,y)=\delta(x-y),\qquad x\in\overline X,\ y\in\overline X,
\end{equation}
where $\delta$ is the Dirac delta. Intuitively, $G$ is the kernel of the right inverse of $\Lop$, in the sense that the solution satisfies $u(x)=\int_{\overline X} G(x,\eta)\,[g, b](\eta)\,\dvol(\eta)$, $\forall x \in \overline X$, where $\dvol$ denotes the volume measure on $X$ and the surface measure on $\partial X$.
If the system~\eqref{eq:problem_relaxed} admits a Green's function $G$, the prior on the forcing term induces the prior $u\sim\mathcal{GP}(0,k_G)$ on the solution~\citep{sarkka2011,cockayne2017probabilistic}, where
\begin{equation}\label{eq:kG}
    k_G(x,y)=\int_X\!\int_X G(x,\zeta)\,\Lambda(\zeta,\eta)\,G(\eta,y)\,\dvol(\zeta)\,\dvol(\eta),
\end{equation}
and we denote by $H_{k_G}(X)$ its reproducing kernel Hilbert space (RKHS). ~\citet{cockayne2017probabilistic} show that $u_0 \in H_{k_G}(X)$.
Another useful observation is that $g\sim\mathcal{GP}(0,\Lambda)$ if and only if $u\sim\mathcal{GP}(0,k_G)$~\citep[Proposition~2]{cockayne2017probabilistic}, so that the prior can equivalently be placed on the forcing term $g$ or on the solution $u$.

Since Green's functions are rarely available in closed form, we rely instead on priors placed directly on $u$, which is justified by the previous observation. Collocation methods posit a kernel $k_{\dagger}$ for the solution, $u \sim \mathcal{GP}(0, k_{\dagger})$, chosen so that $H_{k_G}(X)$ is embedded in $H_{k_{\dagger}}(X)$~\citep{cockayne2017probabilistic}. A technical obstruction remains, since a Gaussian measure assigns zero mass to its own RKHS. This is resolved by introducing the \emph{integral-type kernel} of $k_{\dagger}$~\citep[Lemma~2.1]{cialenco2012approximation},
\begin{equation}\label{eq:cialenco}
    k(x,x'):=\int_X k_{\dagger}(x,z)\,k_{\dagger}(z,x')\,\dvol(z),
\end{equation}
which guarantees that $\mathcal{GP}(0,k)$ is supported on $H_{k_{\dagger}}(X)$ and $H_{k}(X)$ is densely embedded in it~\citep[Proposition~4]{cockayne2017probabilistic}.
The main drawback of this choice is that samples of $u\sim\mathcal{GP}(0,k)$ are not constrained to satisfy boundary conditions.
To simplify notation, we write $\Lop:=[\Aop,\Bop]^\top$ for the stacked operator. Let $Z_{N_c}=\{ x_1,\dots, x_{N_c}\} \subset \overline X$ be the \emph{collocation set}, at which the forcing term is observed,  so that $\bm y=[g,b](Z_{N_c})\in\R^{N_c}$ collects the values of $g$ at the interior points and of $b$ at the boundary points of $Z_{N_c}$. Denoting by $\Lop_x$ and $\Lop_y$ the action of $\Lop$ on the first and second argument of a kernel, respectively, the posterior process is $\Pi^{u\mid\bm y}=\mathcal{GP}(\mu,\Sigma)$, with
\begin{align}
    &\mu(x)= \bigl[\Lop_y k(x,Z_{N_c})\bigr]\,
              \bigl[\Lop_x\Lop_y k(Z_{N_c},Z_{N_c})\bigr]^{-1}\bm y,
              \label{eq:postmean}\\
    &\Sigma(x,x') = k(x,x')-\bigl[\Lop_y k(x,Z_{N_c})\bigr]\,
              \bigl[\Lop_x\Lop_y k(Z_{N_c},Z_{N_c})\bigr]^{-1}
              \bigl[\Lop_x k(Z_{N_c},x')\bigr],
              \label{eq:postcov}
\end{align}
where $\Lop_x\Lop_y k(Z_{N_c},Z_{N_c})$ is the Gram matrix.
The boundary conditions are thus enforced only at the collocation points placed on $\partial X$. They are therefore not satisfied exactly by the posterior, but the posterior restricted to the boundary converges as the number of boundary points grows.

As noted by~\citet{cialenco2012approximation}, it is convenient to choose $k_{\dagger}$ as a Mercer's kernel. Given its series expansion
\[
    k_{\dagger}(x,y) = \sum_{i=0}^{\infty} \lambda_i \phi_i(x)\phi_i(y),
\]
the integral-type kernel $k$ is obtained by squaring $\lambda_i$ in the series. As detailed by~\citet{cialenco2012approximation}, Matérn kernels are a natural choice since it is easy to control the function space to which samples belong. Let $k_{\dagger}$ be a Matérn with smoothness $\nu$, whose RKHS is norm-equivalent to the Sobolev space
$H^{\nu+d/2}(X)$~\citep[Example~A.1]{cialenco2012approximation}, and assume that the solution lies in $H^{s}(X)$. Squaring the eigenvalues doubles the Sobolev order, so that $H_{k}(X)\cong H^{2\nu+d}(X)$.
Two conditions then constrain $\nu$.
First, $H^{s}(X)\subset H_{k_{\dagger}}(X)$ requires $\nu\le s-d/2$, while $H_{k}(X)$ is dense in $H_{k_{\dagger}}(X)$.
Second, evaluating $\Lop u$ at the collocation points requires $H_{k_{\dagger}}(X)\subset C^{l}(X)$, with $l$ the order of the differential operator, which gives $\nu>l$. Combining the two, the admissible choices are $\nu\in(l,\,s-d/2]$. For the integral-type
kernel, this corresponds to $H_{k}(X)\cong H^{2\nu+d}(X)$ with $2\nu+d\in(2l+d,2s]$.

\subsection{Generalized finite differences on manifolds}\label{sec:gfdm}
We briefly present the GFDM approach of~\citet{Suchde2019} proposed for PDEs on implicit manifolds. This is the discretization we use whenever the differential operator cannot be applied in closed form. Its main advantage over mesh-based methods is that it only requires a point cloud sampled from the manifold.

Using the stacked notation introduced in Section~\ref{sec:PMMs_intro}, consider the problem
\[
\Lop u = [g, b]^{\top},
\]
where $X=\Mfd$ is a smooth compact $d$-dimensional manifold embedded in $\R^{n}$. Let $X_N \subset \Mfd$ be a point cloud of size $N$.
GFDM uses the point cloud to approximate the linear operator $\Lop$ with a square matrix $L$ of dimension $N$ whose $i$-th row reproduces $\Lop$ exactly on the local polynomials of degree at most $p$ centred at $x_i \in X_N$.

At each node $x_i$, we consider an orthonormal frame $T_i=[t_{i,1},\dots,t_{i,d}]\in\R^{n\times d}$ of the tangent space $\mathcal T_{x_i}\Mfd$, with associated orthogonal projector
\begin{equation}\label{eq:projector}
  P_i \;=\; T_iT_i^\top \;=\; \sum_{s=1}^{d}t_{i,s}\,t_{i,s}^\top,
\end{equation}
which is frame independent. The frame can be either known in advance or estimated from the point cloud, for example by tangent principal component analysis~\citep{tyagi2013tangent}. We then define the stencil of $x_i$ as the set $S_i=\{x_{i,1},\dots,x_{i,\mathsf K}\}$ of its $\mathsf K$ nearest neighbours, ordered so that $x_{i,1}=x_i$, and we denote by
\begin{equation}
  z_{i,a}:=T_i^\top\bigl(x_{i,a}-x_i\bigr)\in\R^{d},
\end{equation}
the coordinates of the neighbours in the frame at $x_i$. The space of local polynomials of degree at most $p$ at $x_i$ is
\begin{equation}\label{eq:basis}
  \Pip(x_i)=\spn\bigl(\{q_{i,\beta}\}_{|\beta|\le p}\bigr),
  \qquad
  q_{i,\beta}(x)=\prod_{s=1}^{d}\bigl[t_{i,s}\cdot(x-x_i)\bigr]^{\beta_s},
\end{equation}
and has dimension $Q:=\dim\Pip(x_i)=\binom{p+d}{d}$.

The stencil coefficients, which form the rows of the GFDM matrix, are computed in two steps. In the first step, $u$ is approximated around $x_i$ by a weighted least-squares fit in $\Pip(x_i)$. To this end, we take $\mathsf K>Q$, assign positive weights $\omega_{i,a}$ to the points of $S_i$, and write $u_i=(u_{i,1},\dots,u_{i,\mathsf K})^\top$, with $u_{i,a}=u(x_{i,a})$, for the values of $u$ on the stencil. The fit minimizes the quadratic objective
\begin{equation}
  J_i(\bm\theta)=\sum_{a=1}^{\mathsf K}\omega_{i,a}
  \bigl(\mathfrak{q}_i(x_{i,a})^\top\bm\theta-u_{i,a}\bigr)^2,
  \qquad
  \mathfrak{q}_i(x):=\bigl(q_{i,\beta(1)}(x),\dots,q_{i,\beta(Q)}(x)\bigr)^\top,
\end{equation}
over the coefficient vectors $\bm\theta\in\R^{Q}$, where $\{\beta(1),\dots,\beta(Q)\}$ enumerates the multi-indices with $|\beta|\le p$. Stacking the rows $\mathfrak{q}_i(x_{i,a})^\top$ into the Vandermonde matrix $V_i\in\R^{\mathsf K\times Q}$ with entries $(V_i)_{a\gamma}=z_{i,a}^{\beta(\gamma)}$, and writing $\Omega_i=\diagop(\omega_{i,1},\dots,\omega_{i,\mathsf K})$, the objective reads
\begin{equation}\label{eq:wls}
  J_i(\bm\theta)=(V_i\bm\theta-u_i)^\top\Omega_i(V_i\bm\theta-u_i).
\end{equation}
Provided that $V_i$ has full column rank, the optimum is attained at
\begin{equation}\label{eq:theta}
  \bm\theta_i=\mathsf A_i^{-1}V_i^\top\Omega_i\,u_i,
  \qquad
  \mathsf A_i:=V_i^\top\Omega_iV_i .
\end{equation}
and the local approximation of $u$ around $x_i$ is $\hat u_i:=\mathfrak{q}_i^\top\bm\theta_i\in\Pip(x_i)$.

For the second step we seek the vector $c^{\Lop}_i\in\R^{\mathsf K}$ of stencil coefficients, which collects the nonzero entries of the $i$-th row of $L$, such that $\Lop u(x_i)\approx\sum_a c^{\Lop}_{i,a}u(x_{i,a})$. Replacing $u$ by $\hat u_i$ and using the linearity of $\Lop$,
\begin{equation}\label{eq:b}
  \Lop\hat u_i(x_i)=\sum_{\gamma=1}^{Q}(\bm\theta_i)_\gamma\,\Lop q_{i,\beta(\gamma)}(x_i)
  =\bigl(\bm b^{\Lop}_i\bigr)^\top\bm\theta_i,
  \qquad
  \bigl(\bm b^{\Lop}_i\bigr)_\gamma:=\Lop\,q_{i,\beta(\gamma)}(x_i),
\end{equation}
where $\bm b^{\Lop}_i$ collects the action of $\Lop$ on each basis element, evaluated at the base point. Substituting~\eqref{eq:theta} into~\eqref{eq:b}  yields the stencil coefficients
\begin{equation}\label{eq:stencil}
  \bigl(c^{\Lop}_i\bigr)^\top=\bigl(\bm b^{\Lop}_i\bigr)^\top \mathsf A_i^{-1}V_i^\top\Omega_i .
\end{equation}
Arranging the vectors $c^{\Lop}_i$ into rows gives a sparse matrix $L\in\R^{N\times N}$ with at most $N\mathsf K$ nonzeros entries, and $Lu|_{X_N}\approx\Lop u|_{X_N}$.

The accuracy of this approximation is measured in terms of the fill distance. Given a finite set of points $Y$ contained in a domain $D$, the fill distance of $Y$ in $D$ is $h_{Y,D}:=\sup_{x\in D}\dist(x,Y)$, where $\dist$ is the distance on $D$.
In particular, the fill distance of the node cloud with respect to the geodesic distance $d_g$ is
\begin{equation}
  h_N:=h_{X_N,\Mfd}=\sup_{x\in \Mfd}\min_{j\le N}d_g(x,x_j).
\end{equation}
Fix $h_0$ with $h_N\le h_0$ and let $\Mfd^{*}=\bigcup_{x\in X_N}B(x,h_0)$, where $B(x,r)$ denotes the geodesic ball of radius $r$ centred at $x$. Then, for $f\in C^{p+1}(\Mfd^{*})$ and an operator of order $l\le p$,~\citet{mirzaei2012generalized} shows that
\begin{equation}\label{eq:rate}
  \bigl|\Lop f(x_i)-Lf_{|_{X_N}}(x_i)\bigr|\le C\,h_N^{\,p+1-l}\,|f|_{C^{p+1}(\Mfd^{*})}.
\end{equation}
The approximation is therefore consistent of order $p+1-l$. Replacing the operator by a matrix built on the point cloud introduces an error which vanishes as $h_N\to0$, hence as $N\to\infty$, under mild conditions on the sampling.

\section{Probabilistic meshless methods on manifolds}\label{sec:pmm}
In this section we introduce our extension of probabilistic meshless methods to compact manifold domains.  We distinguish two cases, depending on the information available about the manifold. A manifold is \emph{known} if its tangent spaces and the eigendecomposition of its Laplace--Beltrami operator are available in closed form. Since this is the case only for special manifolds, we also consider the more general case of \emph{implicit} manifolds, for which the only available information is a sample of points in the embedding space.  Known manifolds may have a boundary, whereas implicit manifolds must be closed.
In both cases, we focus on forward problems in which the differential operator is replaced by an approximation, and the resulting discretization error plays the role of the uncertainty on the operator in~\eqref{eq:problem_relaxed}. The extension to inverse problems appears possible along the lines of~\citet{Harlim_2022}, but we do not pursue it here.

\subsection{Known manifolds}
\label{sec:known}
If the manifold is known and $\Lop$ is available in closed form, the method is exactly the probabilistic meshless method built on the integral-type kernel~\eqref{eq:cialenco} transported to the manifold setting. Indeed, the construction of Section~\ref{sec:PMMs_intro} relies on properties which carry over to compact manifolds. First, a Gaussian process on a compact manifold can be built from the eigendecomposition of the Laplace--Beltrami operator~\citep{bogachev1998gaussian, borovitskiy2020matern}. Second, linear operators act on Gaussian processes as in the Euclidean case~\citep{bogachev1998gaussian}. Third, the Sobolev embeddings used to select the kernel smoothness still hold~\citep{munkres2018analysis}. Finally, the density of the RKHS of the integral-type kernel follows from~\citet[Lemma~2.1]{cialenco2012approximation}, whose proof does not rely on any Euclidean structure the manifold does not possess, and thus transfers to manifolds. Since no approximation is involved, this model has the same control as in the Euclidean case. Therefore, it is our reference model, and the aim of the methods that follow is to approximate it.

The reference model above requires $\Lop$ in closed form, which is rarely the case. If the operator $\Lop$ is not available in closed form, we propose to sample a point cloud over the manifold and replace $\Lop$ by its GFDM discretization $L$. More precisely, we first define a Gaussian process  $u \sim \mathcal{GP}(0, k)$ on the manifold with $k$ the integral-type kernel~\eqref{eq:cialenco}. We then sample a point cloud $X_N \subset \Mfd$ of size $N$ over the manifold, and restrict the Gaussian process to it, obtaining a $N$-dimensional Gaussian vector $u_N$. The same point cloud is used to discretize the operator via GFDM approximation, i.e., to build the GFDM matrix $L$. The discrete surrogate of the $\Lop u$ is then the Gaussian vector $Lu_N\sim\mathcal N(0,L\,k(X_N,X_N)\,L^{\top})$.
The posterior is obtained by sampling the forcing term at collocation points, taken as a subset of the point cloud $Z_{N_c} \subseteq X_N$. Denoting by $L_c\in\R^{N_c\times N}$ the rows of $L$ indexed by $Z_{N_c}$, the observations $\bm y$ of the forcing terms at the collocation points are then modelled as $\bm y=L_c u_N$. Conditioning on them gives the posterior distribution $\Pi^{u\mid\bm y}=\mathcal{N}(\mu,\Sigma)$ on the point cloud, with
\begin{align}
    &\mu(x_i)= \bigl[k(x_i,X_N)\,L_c^{\top}\bigr]\,
              \bigl[L_c\,k(X_N,X_N)\,L_c^{\top}\bigr]^{-1}\bm y,\\
    &\Sigma(x_i,x_j) = k(x_i,x_j)-\bigl[k(x_i,X_N)\,L_c^{\top}\bigr]\,
              \bigl[L_c\,k(X_N,X_N)\,L_c^{\top}\bigr]^{-1}
              \bigl[L_c\,k(X_N,x_j)\bigr],
\end{align}
for $x_i, x_j \in X_N$. This construction fits the doubly stochastic framework of Section~\ref{sec:PMMs_intro}. Indeed, the GFDM approximation depends on the point cloud, which can be regarded as a realization of a random variable $Y:\Omega' \to \Mfd$, so that the uncertainty on the operator reflects the different approximations produced by different samples of the point cloud. The following result shows that, at the nodes, the GFDM discretization applied to samples of the prior converges to the exact operator, at a rate governed by the fill distance.
\begin{theorem}[Convergence at the nodes]\label{thm:known}
Let $\Mfd$ be a compact manifold, possibly with boundary, for which the spectral expansion of the Laplace--Beltrami operator is known. Let $\Lop$ be a linear differential operator of order $l$ and $L$ its GFDM approximation of polynomial degree $p\ge l$. Assume the Gaussian process $u$ is such that sample paths lie almost surely in $C^{p+1}(\Mfd)$. Then, for $\PP$-almost every $\omega\in\Omega$,
\[
   \bigl|L\,u_{|X_N}(\cdot,\omega)(x_i)-\Lop u(x_i,\omega)\bigr|\  \leq C\,h_N^{\,p+1-l}\,|u(\cdot,\omega)|_{C^{p+1}(\Mfd^{*})}
   \qquad\text{for every node }x_i\in X_N,
\]
with $C$ independent of $\omega$ and $N$.
\end{theorem}
\begin{proof}
Let $\nullset\subset\Omega$ be the null set outside which sample paths fail to be $C^{p+1}(\Mfd)$, and fix $\omega\notin\nullset$. Writing $f:=u(\cdot,\omega)\in C^{p+1}(\Mfd)$, the consistency property of the GFDM stencil~\eqref{eq:rate} gives
\begin{equation}
   \bigl|Lf(x_i)-\Lop f(x_i)\bigr|\ \le\ C\,h_N^{\,p+1-l}\,|f|_{C^{p+1}(\Mfd^{*})}.
\end{equation}
Since $p\ge l$ the exponent is positive, so that the right hand side goes to zero if $h_N\to0$.
\end{proof}
The resulting model is a probabilistic meshless method on a known manifold, in which the differential operator is replaced by its GFDM approximation. By Theorem~\ref{thm:known}, this approximation is consistent at the nodes with the rate $h_N^{\,p+1-l}$  inherited from the GFDM stencil.
While the probabilistic meshless methods of~\citet{cockayne2017probabilistic} requires $\Lop$ to be applied to the kernel in closed form, it is here sufficient to have a point cloud on the manifold, on which the operator is approximated.

\subsection{Implicit manifolds}
\label{sec:implicit}
In the known case, the kernel can be evaluated exactly and only the differential operator needs to be approximated. In contrast, the only information available about an implicit manifold is a point cloud $X_N \subset \Mfd$ of size $N$, from which both the kernel and the operator of probabilistic meshless methods must be estimated. This requires additional assumptions.

Regarding the Gaussian process, we focus on the class of processes whose kernels are defined through the spectral expansion of the Laplace--Beltrami operator and are expressed using a Fourier series of the form
\[
k(x,y) =\sum_{i=0}^{\infty} \Phi(\lambda_i) \phi_i(x) \phi_i(y),
\]
with $\{\lambda_i, \phi_i\}$ the eigenpairs of the Laplace--Beltrami operator $\Delta_\Mfd$ , and $\Phi$ is a real function. Since this is a Mercer kernel, the corresponding integral-type kernel~\eqref{eq:cialenco} is obtained by replacing $\Phi(\lambda_i)$ with $\Phi(\lambda_i)^2$. For notational simplicity, we assume that this squaring is already included in $\Phi$. As the eigenpairs of $-\Delta_\Mfd$ are not available, we approximate them by those of the graph Laplacian $\mathsf L^{(\alpha)}\in\R^{N\times N}$ built on the point cloud.
On closed manifolds, the eigenvalues and eigenvectors of the graph Laplacian converge uniformly to the eigenvalues and eigenfunctions of $-\Delta_\Mfd$ as $N\to\infty$~\citep{Dunson2021}. On manifolds with boundary, uniform convergence of the eigenvectors is still an open question~\citep{Peoples2026}, and since our analysis relies on it, we restrict to closed manifolds. Further details on the graph Laplacian and the associated convergence rates are given in Appendix~\ref{app:graph}.

Writing $0=\hat\lambda_{0}\le\dots\le\hat\lambda_{N-1}$ for the eigenvalues of $-\mathsf L^{(\alpha)}$ and $\hat\phi_{0},\dots,\hat\phi_{N-1}$ for the associated eigenvectors, \citet{Fichera2023} approximate a Mercer kernel on $\Mfd$ by
\begin{equation}\label{eq:khat}
    \khat(x_i,x_j)=\sum_{r=0}^{N-1}\Phi(\kappa^{2}+\hat\lambda_{r})\,
      \hat\phi_{r}^{(i)}\hat\phi_{r}^{(j)},
\end{equation}
where the superscript $(i)$ denotes the $i$-th entry of the vector. This approximation converges pointwise uniformly to the true kernel on the point cloud as $N\to\infty$~\citep{Dunson2021, Fichera2023}, and we denote this rate as
\begin{equation}\label{eq:etadef}
    \eta_N:=\max_{1\le i,j\le N}\bigl|\khat(x_i,x_j)-k(x_i,x_j)\bigr| .
\end{equation}
Note that, unlike in the known case, the Gaussian process is now only defined on the point cloud, i.e., $u_N \sim \mathcal{N} (0, \hat K_N)$ with $\hat K_N:=\khat(X_N,X_N)$.

Since the graph-based kernel~\eqref{eq:khat} is only defined on the point cloud, and the results of~\citep{Dunson2021, Fichera2023} only guarantee its convergence at the nodes, the natural target is the convergence of the Gram matrix at the nodes. Specifically, given a certain operator approximation $L$, we aim to prove that $|\Lop_x\Lop_y k - L \,\hat k \,L^{\top}| \to 0$ as $N$ grows.
Replacing $\Lop$ by the standard GFDM matrix, as in the known case, does not guarantee this convergence in general. The reason is that the stencil coefficients of the standard construction have $\ell^1$ norm of order $h_N^{-l}$, so that applying the GFDM matrix on both sides of $\khat$ amplifies the kernel error $\eta_N$ by a factor $h_N^{-2l}$. Past a certain $N$, this amplification dominates the gain in consistency.

To overcome this issue, we propose a different GFDM approximation in which the consistency rate can be slowed down to match the graph kernel rate, so that the growth of the stencil norm is controlled. This is achieved by setting a radius $\varrho_N$, such that $h_N/\delta_1 \leq \varrho_N \leq \min\{1, \inj/2\}$, where $\inj$ is the injectivity radius of the manifold $\Mfd$ and $\delta_1$ is a constant independent of $N$. The stencil at $x_i$ is then the set $S^{\varrho_N}_i:=X_N\cap B(x_i,\varrho_N)$ of all nodes in the geodesic ball of radius $\varrho_N$. The stencil coefficients $c_{ij}^{\varrho_N}$ are, among those that reproduce $\Lop$ on local polynomials, the ones of minimal $\ell^1$ norm, as detailed in Section~\ref{sec:implementation} and Appendix~\ref{app:implicit}. We denote the resulting matrix by $L^{\varrho_N}$, to distinguish it from the standard GFDM matrix $L^h$. We show that $L^{\varrho_N}$ retains the approximation properties of $L^h$, with a consistency order governed by $\varrho_N$, meaning that it can be tuned through the choice of radius. More details about this approximation are provided in Appendix~\ref{app:implicit}.

The main result of this section states that, under growth conditions on $\varrho_N$ and the four assumptions below, the approximated Gram matrix converges uniformly to the true one as the number of points increases.
\begin{assumption}[Geometry and sampling]\label{ass:geom}
$\Mfd$ is a smooth closed Riemannian $d$-manifold. Define $\mu=\rho\,\dvol$ with $\rho\in C^{\infty}(\Mfd)$ and $0<\rho_{\min}\le\rho\le\rho_{\max}$. The nodes $x_1,\dots,x_N$ are i.i.d.\ $\sim\mu$.
\end{assumption}
\begin{assumption}[Operator]\label{ass:op}
$\Lop=\sum_{|\beta|\le l}a_\beta\partial^{\beta}$ in normal coordinates with $a_\beta\in C^{p+1}(\Mfd)$, and the stencil degree satisfies $p\ge l$. Set $A_\infty:=\max_{|\beta|\le l}\|a_\beta\|_{L^{\infty}}$ and $C_{\Lop}:=A_\infty\binom{l+d}{d}$.
\end{assumption}
\begin{assumption}[Prior]\label{ass:prior}
Let $-\Delta_\Mfd$ have eigenvalues $0=\lambda^{\Delta}_0\le\lambda^{\Delta}_1\le\cdots$ with $L^{2}(\Mfd,\dvol)$-orthonormal eigenfunctions $\phi_0,\phi_1,\dots$. Fix $\kappa>0$ and a decreasing $\Phi:(0,\infty)\to(0,\infty)$ with $\Phi(t)\asymp t^{-\mathfrak s}$, and let
\begin{equation}
    k(x,y):=\sum_{r\ge0}\Phi\bigl(\kappa^{2}+\lambda^{\Delta}_r\bigr)\phi_r(x)\phi_r(y).
\end{equation}
We require $\mathfrak s>p+1+d$. Moreover, we assume also that $\|k\|_{C^{p+1,p+1}}<\infty$.
\end{assumption}
\begin{assumption}[Graph kernel]\label{ass:graph}
$\khat\in\R^{N\times N}$ is the symmetric matrix from the evaluation of \eqref{eq:khat} on the point cloud, and $\eta_N$, defined in \eqref{eq:etadef}, satisfies $\eta_N \xrightarrow[N\to\infty]{\PP'} 0$ but slowly enough that $h_N^{\,p+1+l}/\,\eta_N\to0$.
\end{assumption}
\noindent Note that Assumption~\ref{ass:graph} holds, for instance, for graph heat kernels \citep{Dunson2021} and graph Mat\'ern kernels \citep{Fichera2023}.

Each of the four assumptions has a simple interpretation. In Assumption~\ref{ass:geom}, the manifold must be closed because graph Laplacian eigenfunctions have no uniform convergence guarantee at the boundary, and the density is bounded so that no region is left without nodes, which ensures that $h_N$ decreases predictably. In Assumption~\ref{ass:op}, the smoothness of the coefficients makes the action of $\Lop$ on local polynomials exact and bounded, and the condition $p\ge l$ guarantees the stencil basis is rich enough to reproduce it. The decay imposed in Assumption~\ref{ass:prior} makes the kernel $C^{p+1}$ in each argument. Finally, Assumption~\ref{ass:graph} requires the graph kernel to converge, but not faster than $h_N^{p+1+l}$, since otherwise the optimal radius would fall below the fill distance and the stencil could not be used.

To compare the exact and approximate models, we introduce the following notation
\begin{equation}\label{eq:gram}
    K(x,y):=\bigl(\Lop_x\Lop_y k\bigr)(x,y),
    \qquad K_{ij}:=K(x_i,x_j),
    \qquad \Khat:=L^{\varrho_N}\,\khat\,(L^{\varrho_N})^{\top},
\end{equation}
where $K$ is the Gram kernel induced by the operator and the prior, $(K_{ij})_{i,j}$ is its restriction to the point cloud, and $\Khat$ is the discrete Gram matrix, whose dependence on $\varrho_N$ is left implicit.

\begin{theorem}[Uniform convergence of the discrete Gram matrix]\label{thm:main}
Under~\ref{ass:geom}--\ref{ass:graph}, let $(\varrho_N)_N$ be admissible in the sense of Lemma~\ref{lem:stencil} and such that
\begin{equation}\label{pdem:eq:rhocond}
    \varrho_N\longrightarrow 0
    \qquad\text{and}\qquad
    \frac{\eta_N}{\varrho_N^{2l}}\longrightarrow 0 ,
\end{equation}
as $N \to \infty$.
Then, with probability tending to one,
\begin{equation}\label{pdem:eq:gramrate}
    \max_{1\le i,j\le N}\bigl|\Khat_{ij}-K_{ij}\bigr|
    \ \le\ C\left(\varrho_N^{\,p+1-l}+\frac{\eta_N}{\varrho_N^{2l}}\right)
    \ \longrightarrow\ 0
\end{equation}
as $N\to\infty$, with $C$ independent of $N$. Moreover, the optimal rate is attained for
$\varrho_N\asymp\eta_N^{1/(p+1+l)}$, which gives
\begin{equation}\label{pdem:eq:gramopt}
    \max_{1\le i,j\le N}\bigl|\Khat_{ij}-K_{ij}\bigr|
    \ \le\ C\,\eta_N^{\frac{p+1-l}{p+1+l}} .
\end{equation}
\end{theorem}
The proof is given in Appendix~\ref{app:implicit}.
The two terms in~\eqref{pdem:eq:gramrate} reflect the trade-off described above: the first is the consistency error of $L^{\varrho_N}$, which decreases with the radius, while the second is the kernel error amplified by the stencil coefficients, which increases as the radius shrinks. The optimal radius $\eta_N^{1/(p+1+l)}$ is admissible, that is, larger than a multiple of $h_N$, precisely when $h_N^{\,p+1+l}/\eta_N\to0$, which explains the condition in Assumption~\ref{ass:graph}. Finally, we note that convergence of the Gram matrix at the nodes does not imply weak convergence of the corresponding Gaussian processes as measures on a function space~\citep{bogachev1998gaussian}. This is, however, a standard limitation in the implicit manifold setting~\citep{Dunson2021, Fichera2023}.

\section{Error bounds for the posterior mean}\label{sec:contraction}

When a Gaussian process is used for function approximation, it is of interest to quantify how fast the posterior mean converges to the true solution as the number of collocation points increases. In this section, we bound the error of the posterior mean at the nodes, in a form which separates the error intrinsic to PMMs from the error introduced by the approximations of Section~\ref{sec:pmm}.

We define $v(x) \in \R^{N_c}$ and $A\in\R^{N_c\times N_c}$ with entries
\[ v_j(x) = \Lop_yk(x,y)|_{y=z_j}, \qquad
   A_{ij} =\Lop_x\Lop_yk(x,y)|_{x=z_i, y=z_j},
\]
so that $A_{ij}=K(z_i,z_j)$ in the notation of~\eqref{eq:gram}. With this notation, the posterior mean and covariance of the exact model read $\mu(x)=v(x)^{\top}A^{-1}\bm y$ and $\sigma(x)^{2}=k(x,x)-v(x)^{\top}A^{-1}v(x)$, respectively.
We denote by $\tilde\cdot$ the corresponding quantities computed from the approximations at hand. For instance, we write $\tilde k=k$ when the manifold is known, and $\tilde k=\khat$ in the implicit case, while $\tilde v$ and $\tilde A$ are obtained by replacing $\Lop$ with the GFDM matrix in both cases.
The Lebesgue constant is $\Lambda_{\mathrm{Leb}}(x):=\max\bigl\{\|A^{-1}v(x)\|_{1},\|\tilde{A}^{-1}\tilde{v}(x)\|_{1}\bigr\}$. The size of the perturbation is measured by
\begin{equation}\label{eq:epsdef}
   \epsilon_N:=\max\Bigl\{\ \max_{x\in X_N}\bigl|\tilde k(x,x)-k(x,x)\bigr|,\
   \max_{x\in X_N}\max_j\bigl|\tilde v_j(x)-v_j(x)\bigr|,\
   \max_{i,j}\bigl|\tilde A_{ij}-A_{ij}\bigr|\ \Bigr\}.
\end{equation}
Lemma~\ref{lem:epsilon} bounds $\epsilon_N$ in the known and in the implicit case, building on Theorems~\ref{thm:known} and~\ref{thm:main}, respectively.

The following theorem bounds the distance between the approximated posterior mean and the true solution at the nodes, separating the collocation error, which is intrinsic to PMMs, from the error introduced by the two approximations.
\begin{theorem}[Error bound for the posterior mean]\label{thm:contraction}
Assume that $k$ is such that $H_{k_G}(\Mfd)\hookrightarrow H_k(\Mfd)$, which in particular implies $u_0\in H_k(\Mfd)$. Let $\epsilon_N$ be as in Lemma~\ref{lem:epsilon}. For every node $x\in X_N$,
\begin{equation}\label{eq:contraction}
   \bigl|\tilde\mu(x)-u_0(x)\bigr|
   \ \le\ \Bigl(\sigma(x)^{2}+2\epsilon_N\bigl(1+\Lambda_{\mathrm{Leb}}(x)\bigr)^{2}\Bigr)^{1/2}
     \|u_0\|_{H_k},
\end{equation}
and consequently, using $\sqrt{a^{2}+b}\le a+\min\{\sqrt b,\;b/(2a)\}$,
\[
   \bigl|\tilde\mu(x)-u_0(x)\bigr|
   \ \le\ \Bigl[\sigma(x)
     +\min\Bigl\{\sqrt{2\epsilon_N}\,\bigl(1+\Lambda_{\mathrm{Leb}}(x)\bigr),\
       \tfrac{\epsilon_N(1+\Lambda_{\mathrm{Leb}}(x))^{2}}{\sigma(x)}\Bigr\}\Bigr]
     \|u_0\|_{H_k}.
\]
Moreover, under the assumptions of Lemma~\ref{lem:sigma}, $\sigma(x)=O\bigl(h_{N_c}^{\,\mathfrak s-l-d/2}\bigr)$, and thus the bound reads
\begin{equation}\label{eq:contraction_comparable}
   \bigl|\tilde\mu(x)-u_0(x)\bigr|
   =O\Bigl(h_{N_c}^{\,\mathfrak s-l-d/2}
     +\sqrt{\epsilon_N}\,\bigl(1+\Lambda_{\mathrm{Leb}}(x)\bigr)\Bigr)\|u_0\|_{H_k}.
\end{equation}
\end{theorem}
The proof is given in Appendix~\ref{app:contraction}. The bound~\eqref{eq:contraction_comparable} separates the two sources of error. The first term is the collocation error intrinsic to PMMs, for which, setting $\epsilon_N=0$ recovers the estimate of~\citep[Proposition~6]{cockayne2017probabilistic}. The second term is the price of approximating the operator, in the known case, or both the kernel and the operator, in the implicit case, and it vanishes as $N\to\infty$ by Lemma~\ref{lem:epsilon}, provided that $\sqrt{\epsilon_N}\,\Lambda_{\mathrm{Leb}}(x)\to0$.

\section{Implementation}
\label{sec:implementation}
The main practical difficulty in implementing $L^{\varrho_N}$ is that the admissible range of stencil radii $\varrho_N$ depends on geometric quantities which are unknown on an implicit manifold.  We first introduce these quantities, and then describe how to estimate them. For $x\in\Mfd$, the exponential map $\exp_x:\mathcal T_x\Mfd\to\Mfd$ sends a tangent vector $z$ to the point reached at time one along the geodesic leaving $x$ with velocity $z$. It is a diffeomorphism onto its image on any ball of radius at most $\inj(x)$, and the injectivity radius $\inj:=\inf_{x\in\Mfd}\inj(x)$ is positive on a compact manifold. The inverse of the exponential map at a node,
\[
\xi_i:=\exp_{x_i}^{-1}:B(x_i,\inj)\longrightarrow\mathsf B(0,\inj)\subset\R^{d},
\]
defines the normal chart at $x_i$, where $B$ denotes a geodesic ball in $\Mfd$ and $\mathsf B$ a Euclidean ball in $\R^{d}$. The chart distortion constant is the smallest $C_\inj\ge1$ such that
\begin{equation}
C_\inj^{-1}d_g(y,y’)\ \le\ |\xi_i(y)-\xi_i(y’)|\ \le\ C_\inj d_g(y,y’),
\end{equation}
for all $i$ and all $y,y'\in B(x_i,\inj/2)$. It is finite on a compact manifold, depends only on $\Mfd$, and equals one if and only if $\Mfd$ is flat.

By Lemma~\ref{lem:stencil}, the stencil radius must satisfy
\begin{equation}
\frac{h_N}{\delta_1}\ \le\ \varrho_N\ \le\ \min\Bigl\{1,\ \frac{\inj}{2}\Bigr\},
\qquad \delta_1=\frac{c_p}{3C_\inj},
\end{equation}
where $c_p=3/(4(2+\sqrt3)p^{2})$ and $p$ is the polynomial degree of the approximation.
Since $\inj$, $C_\inj$ and $h_N$ are unknown on an implicit manifold, this admissible range must be estimated.
We start with the upper bound, $\varrho_N\le\inj/2$. For a smooth closed manifold, the reach $\tau_\Mfd$ is defined as
\begin{equation}\label{eq:reach}
\tau_{\Mfd}:=\inf_{w\in\Mfd}\ \inf_{w'\in\Mfd\setminus\{w\}}
\frac{\|w'-w\|^{2}}{2\,\dist(w'-w,\,\mathcal T_w\Mfd)}.
\end{equation}
The reach bounds the injectivity radius from below, $\pi\tau_{\Mfd}\le\inj$, and it can be estimated from the point cloud~\citep{aamari2019estimating}, which provides a computable upper bound for $\varrho_N$. Explicitly, the estimate is obtained by solving
\begin{equation}\label{eq:estreach}
    \hat \tau_{\Mfd} = \inf_{x \neq y \in X_N} \frac{\|y - x\|^2}{2 \dist(y-x, \mathcal T_x\Mfd)}.
\end{equation}
Since the infimum is taken over a finite sample, $\hat \tau_{\Mfd}$ may overestimate $\tau_{\Mfd}$, in which case the resulting upper bound is optimistic.

We turn to the lower bound $\varrho_N \ge\ \frac{h_N}{\delta_1}$, and start by estimating the chart distortion constant $C_{\inj}$. The sectional curvature $\mathrm{Sec}_{\Mfd}$ satisfies~\citep{aamari2019estimating}
\[
    -\frac{2}{\tau^2_{\Mfd}} \leq \mathrm{Sec}_{\Mfd} \leq \frac{1}{\tau^2_{\Mfd}},
\]
so that $|\mathrm{Sec}_{\Mfd}|\le\mathfrak K:=2\tau^{-2}_{\Mfd}$.
A standard application of Rauch's comparison theorem~\citep[Eq.~(4.2)]{Fefferman2020} then gives, for every pair of points $x,y$ lying in a geodesic ball of radius $r$ centered at $x_i$,
\[
   \frac{\sin\bigl(\sqrt{\mathfrak K}\,r\bigr)}{\sqrt{\mathfrak K}\,r}
   \ \le\ \frac{d_g(x,y)}{\bigl|\xi_i(x)-\xi_i(y)\bigr|}
   \ \le\ \frac{\sinh\bigl(\sqrt{\mathfrak K}\,r\bigr)}{\sqrt{\mathfrak K}\,r},
\]
whenever $r\le\frac12\min\{\pi\mathfrak K^{-1/2},\inj\}$. Comparing these inequalities with the definition of the chart distortion constant yields the two candidate values $\sqrt{\mathfrak K}\,r/\sin(\sqrt{\mathfrak K}\,r)$ and $\sinh(\sqrt{\mathfrak K}\,r)/(\sqrt{\mathfrak K}\,r)$. Since $t/\sin t\ge\sinh t/t$ for $t\in(0,\pi)$, the former is the largest, and therefore
\[
   C_{\inj}\ \le\ \frac{\sqrt{\mathfrak K}\,r}{\sin\bigl(\sqrt{\mathfrak K}\,r\bigr)} .
\]
It remains to determine the largest $r$ for which this bound holds.
Since $\mathfrak K=2\tau^{-2}_{\Mfd}$ gives $\pi\mathfrak K^{-1/2}=\pi\tau_{\Mfd}/\sqrt2$, and $\pi\tau_{\Mfd}\le\inj$,
\[
    \pi\mathfrak K^{-1/2}=\frac{\pi\tau_{\Mfd}}{\sqrt2}<\pi\tau_{\Mfd}\le\inj ,
\]
so that the minimum is attained by the curvature term, the injectivity radius plays no role, and the bound holds for $r\le\pi\tau_{\Mfd}/(2\sqrt2)$.
At $r=\pi\tau_{\Mfd}/(2\sqrt2)$, we have $\sqrt{\mathfrak K}\,r=\pi/2$, which gives
\[
   C_{\inj}\ \le\ \frac{\pi/2}{\sin(\pi/2)}\ =\ \frac{\pi}{2}.
\]
Moreover, $\pi\tau_{\Mfd}/(2\sqrt2)\le\inj/2$ follows from $\pi\tau_{\Mfd}\le\inj$. Collecting these requirements, the stencil radius must lie in the interval
\[
   \frac{3 h_N \pi}{2c_p}\ \le\ \varrho_N\ \le\ \min\Bigl\{1,\ \frac{\pi\,\hat\tau_{\Mfd}}{2\sqrt2}\Bigr\}.
\]
Several approaches are available to estimate $h_N$. Under uniform sampling, $h_N= O\bigl((\log N/N)^{1/d}\bigr)$, but the multiplicative constant is not accessible. A computable alternative is to build a nearest-neighbors graph and estimate $h_N$ from the distance between neighbouring nodes, with the drawback that this underestimates the fill distance.

It is important to highlight that an admissible radius requires $N$ to be significantly high. Consider for example the sphere, for which $\tau_{\Mfd}=1$. In this case, using the bound $C_{\inj}\le\pi/2$, the upper bound $\pi\tau_\Mfd/(2\sqrt2)$ exceeds one, and the admissible range reads
\[
    h_N 2 \pi(2 + \sqrt 3 )p^2\ \le\ \varrho_N\ \le\ 1 .
\]
Moreover, if the sampling is quasi-uniform, then $h_N\approx2N^{-1/2}$, which gives
\[
    N \geq 16\pi^2 \, (2 + \sqrt3\,)^{2}p^4
\]
where the right hand side is of the order of $10^4$ for $p=2$. This requirement is less restrictive than it appears, since $L^{\varrho_N}$ is needed precisely when the stencil norm of $L^{h}$ becomes large, which happens at large $N$.

Since the admissible range may be empty for moderate $N$, we use two fallbacks. The first fallback applies when the lower bound $3 h_N \pi/2c_p$ exceeds the upper bound $\pi\hat\tau_{\Mfd}/(2\sqrt2)$. In this case, we set $\varrho_N$ at the ceiling $\min\{1,\pi\hat\tau_{\Mfd}/(2\sqrt2)\}$, accepting that the density requirement of Lemma~\ref{lem:stencil} is violated. The second fallback applies when the upper bound is itself too small to carry a stencil. Since the space of polynomials of degree $p$ has dimension $\binom{p+d}{d}$, a ball containing fewer points cannot, in general, carry coefficients satisfying~\eqref{eq:exactness}. In this case, we exceed the upper bound and take the smallest radius for which $L^{\varrho_N}$ is well defined.

When the admissible range is nonempty, the radius must also satisfy $\eta_N/\varrho_N^{2l}\to0$ in Theorem~\ref{thm:main}. Since the optimal rate is obtained with a radius $\varrho_N^{\ast}\asymp\eta_N^{1/(p+1+l)}$, we estimate $\eta_N$ from the point cloud. Let $\khat'$ be the graph kernel~\eqref{eq:khat} built on the same nodes with bandwidth $2^{1/d}\zeta$ instead of $\zeta$ (Appendix~\ref{app:graph}), that is, the bandwidth that a rule $\zeta\propto N^{-1/d}$ assigns to $N/2$ nodes, and let $\eta'_N$ be its error. We set
\[
   \hat\eta_N:=\max_{1\le i,j\le N}\bigl|\khat(x_i,x_j)-\khat'(x_i,x_j)\bigr|,
   \qquad
   \varrho_N:=\Bigl(\frac{2l\,\hat\eta_N}{p+1-l}\Bigr)^{\frac{1}{p+1+l}},
\]
where $\varrho_N$ is the minimizer of the bound of Theorem~\ref{thm:main} with unit constants (Appendix~\ref{app:implicit}), projected onto the admissible range. By the triangle and the reverse triangle inequality we have
\[
   \bigl|\eta_N-\eta'_N\bigr|\ \le\ \hat\eta_N\ \le\ \eta_N+\eta'_N .
\]
The upper bound gives $\hat\eta_N\to0$ when both kernels converge, hence $\varrho_N\to0$. The lower bound gives $\hat\eta_N\ge\gamma\,\eta_N$ whenever passing from $\zeta$ to $2^{1/d}\zeta$ changes the error by a fixed fraction, that is, $|\eta'_N-\eta_N|\ge\gamma\,\eta_N$ for some $\gamma>0$, and then $\eta_N/\varrho_N^{2l}\le C\,\eta_N^{(p+1-l)/(p+1+l)}\to0$. Both conditions of Theorem~\ref{thm:main} are therefore met, and if moreover $\eta'_N\le C\eta_N$, then $\hat\eta_N\asymp\eta_N$ and the optimal rate~\eqref{pdem:eq:gramopt} is attained.
In practice, to speed up computations, the maximum is taken over a random subset of nodes, and the kernels are normalized so that $\max_{i,j}|\khat(x_i,x_j)|=1$.
Given the radius, the coefficients at $x_i$ are the ones of minimal $\ell^1$ norm among those supported in $S^{\varrho_N}_i$ that satisfy~\eqref{eq:exactness}, computed by linear programming.

\section{Experiments}
\label{sec:experiments}

We validate the proposed methods in two settings. The first is a manufactured advection--diffusion--reaction problem on the circle, the sphere and the flat torus. On these manifolds, the exact solution is available in closed form, so that the theoretical results of Section~\ref{sec:pmm} can be tested directly. The second setting is the smoothing of cortical thickness data on a cortical surface, where the manifold is only available through a triangulated surface and no closed-form solution exists, so that our methods are compared against established discretizations. In both settings, we report the accuracy against a reference solution and, for the probabilistic methods, the calibration of the credible intervals.

\subsection{Advection--diffusion--reaction on known, closed manifolds}

In this section, we consider the advection--diffusion--reaction equation
\begin{equation}\label{eq:advdiff}
    \Lop u\ :=\ \mathfrak r\,u\;+\;\mathfrak V\cdot\grad u\;-\;\mathfrak d\,\Delta u
    \ =\ f ,
\end{equation}
where $\mathfrak r>0$ is the reaction rate, $\mathfrak V\in\mathfrak X(\Mfd)$ is a smooth vector field describing the transport, and $\mathfrak d>0$. We consider three manifolds which are standard in the literature, namely the circle $\mathbb S^{1}\subset\R^2$, the unit sphere $\mathbb S^{2}\subset\R^{3}$ and the flat torus $\mathbb T^2 \subset \R^4$.
The experiment on the circle is designed to expose the stability of $L^{\varrho_N}$ compared to $L^h$ as $N$ increases, as predicted by our theoretical results. The experiments on the sphere and the torus instead assess the quality of the proposed models in a setting closer to a real application.

The problem is constructed in the same way on the three manifolds. We choose $\mathfrak V$ so that the advection term takes a simple form, and we expand in the eigenfunctions of the Laplace--Beltrami operator. The diffusion term scales each eigenfunction by its eigenvalue, while the advection term maps it to a combination of eigenfunctions with the same eigenvalue, so that the action of $\Lop$ on each basis element is available in closed form. Defining $u$ as a finite combination of basis elements therefore gives $f$ in closed form, and the problem is fully known.

On each manifold, we compare four probabilistic configurations, which differ in the information they use about the manifold. The two known configurations use the exact kernel with analytic tangent frames, combined either with the exact operator $\Lop$ or with $L^{h}$. The two implicit configurations use the graph kernel with tangent frames estimated by tangent PCA~\citep{tyagi2013tangent}, combined either with $L^{h}$ or with $L^{\varrho_N}$. In the implicit case, $L^{h}$ is rebuilt on the estimated frames, so that no configuration pairs an estimated kernel with the exact geometry. As a non-probabilistic baseline, we also report a direct GFDM solve with $L^{h}$ on the analytic frames.

Without further regularity assumptions,  a strong solution of~\eqref{eq:advdiff} lies in $C^{2}(\Mfd)$. We impose the smoothness directly on the integral-type kernel $k$, which we take to be of Mat\'ern class of smoothness $\nu$, so that $\mathfrak s=\nu+d/2$. Since sample paths of $\mathcal{GP}(0,k)$ lie almost surely in $C^{m}(\Mfd)$ whenever $\mathfrak s>m+d$, the $C^{2}$ requirement  translates into $\nu>2+d/2$, while Theorem~\ref{thm:known} requires sample paths in $C^{p+1}(\Mfd)$, hence $\nu>p+1+d/2$. We set $\nu=11/2$ throughout, which satisfies both requirements in all our experiments.
Kernel hyperparameters are estimated by maximum likelihood.

\paragraph{Circle:}
Instead of standard Euclidean coordinates, we parametrize points using arclength $\vartheta \in [0, 2\pi)$. The vector field is set as a rotation of the form $\mathfrak V=\mathfrak w \partial_{\vartheta}$. The eigenpairs of $-\Delta$ are the standard Fourier modes $C_{\mathfrak n}=\cos(\mathfrak n\vartheta)$ and $S_{\mathfrak n}=\sin(\mathfrak n\vartheta)$ with eigenvalue $\mathfrak n^{2}$. With our choice of transport field, differentiation in $\vartheta$ gives
\begin{equation*}
    \frac{\partial}{\partial\vartheta}C_{\mathfrak n}
      =-\,\mathfrak n\,S_{\mathfrak n},
    \qquad
    \frac{\partial}{\partial\vartheta}S_{\mathfrak n}
      =+\,\mathfrak n\,C_{\mathfrak n},
\end{equation*}
and combining with the eigenvalue relation gives the closed form
\begin{align*}
    \Lop C_{\mathfrak n}
      &=\bigl(\mathfrak r+\mathfrak d\,\mathfrak n^{2}\bigr)C_{\mathfrak n}
        \;-\;\mathfrak w\,\mathfrak n\,S_{\mathfrak n},\\
    \Lop S_{\mathfrak n}
      &=\bigl(\mathfrak r+\mathfrak d\,\mathfrak n^{2}\bigr)S_{\mathfrak n}
        \;+\;\mathfrak w\,\mathfrak n\,C_{\mathfrak n}.
\end{align*}
Fixing a finite set of integers $\mathcal S$, we define
\[
    u\ :=\ \sum_{\mathfrak n\in\mathcal S}
           \bigl(\mathfrak a_{\mathfrak n}\,C_{\mathfrak n}
                +\mathfrak b_{\mathfrak n}\,S_{\mathfrak n}\bigr),
\]
for coefficients $\mathfrak a_{\mathfrak n},\mathfrak b_{\mathfrak n}$.

\paragraph{Sphere:}
We take for $\mathfrak V$ the solid-body rotation of angular velocity $\mathfrak w\in\R$ about a fixed unit axis $\mathfrak e\in\R^{3}$,
\[
    \mathfrak V(x)\ :=\ \mathfrak w\,(\mathfrak e\times x),\qquad x\in\mathbb S^{2}.
\]
We introduce the spherical coordinates $(\chi,\upsilon)$ about $\mathfrak e$, so that $x=(\sin\chi\cos\upsilon,\ \sin\chi\sin\upsilon,\ \cos\chi)$. The associated orthonormal frame of $\mathcal T_x\mathbb S^{2}$ is $(\hat e_{\chi},\hat e_{\upsilon})$ with $\hat e_{\upsilon}=(-\sin\upsilon,\cos\upsilon,0)$. In this frame the gradient reads
\[
\grad=\hat e_{\chi}\,\partial_{\chi}+\frac{1}{\sin\chi}\,\hat e_{\upsilon}\,\partial_{\upsilon}.
\]
A direct computation gives $\mathfrak e \times x=\sin\chi\,\hat e_{\upsilon}$, so that the advection term becomes
\begin{equation*}
    \mathfrak V\cdot\grad\ =\ \mathfrak w\,\frac{\partial}{\partial\upsilon} .
\end{equation*}
The eigenfunctions of $-\Delta$ on $\mathbb S^{2}$ are the real spherical harmonics $Y_{\mathfrak n}^{\mathfrak m}$, indexed by a degree $\mathfrak n\ge0$ and an order $-\mathfrak n\le\mathfrak m\le\mathfrak n$. For $\mathfrak m\ge1$, these can be written as
\[
    Y_{\mathfrak n}^{+\mathfrak m}
      =\mathfrak c_{\mathfrak n\mathfrak m}\,
        P_{\mathfrak n}^{\mathfrak m}(\cos\chi)\,\cos(\mathfrak m\,\upsilon),
    \qquad
    Y_{\mathfrak n}^{-\mathfrak m}
      =\mathfrak c_{\mathfrak n\mathfrak m}\,
        P_{\mathfrak n}^{\mathfrak m}(\cos\chi)\,\sin(\mathfrak m\,\upsilon),
\]
with $Y_{\mathfrak n}^{0}=\mathfrak c_{\mathfrak n0}P_{\mathfrak n}^{0}(\cos\chi)$. Here, $P_{\mathfrak n}^{\mathfrak m}$ is the associated Legendre polynomial and $\mathfrak c_{\mathfrak n\mathfrak m}$ the $L^{2}(\mathbb S^{2})$ normalization constant. With our choice of transport field, the application of the advection term to spherical harmonics gives
\begin{equation*}
    \frac{\partial}{\partial\upsilon}Y_{\mathfrak n}^{+\mathfrak m}
      =-\,\mathfrak m\,Y_{\mathfrak n}^{-\mathfrak m},
    \qquad
    \frac{\partial}{\partial\upsilon}Y_{\mathfrak n}^{-\mathfrak m}
      =+\,\mathfrak m\,Y_{\mathfrak n}^{+\mathfrak m} .
\end{equation*}
Combining this with the eigenvalue relation $-\Delta Y_{\mathfrak n}^{\pm\mathfrak m}=\mathfrak n(\mathfrak n+1)\,Y_{\mathfrak n}^{\pm\mathfrak m}$, the action of $\Lop$ on $Y_{\mathfrak n}^{+\mathfrak m},Y_{\mathfrak n}^{-\mathfrak m}$ has the closed form
\begin{align*}
    \Lop Y_{\mathfrak n}^{+\mathfrak m}
      &=\bigl(\mathfrak r+\mathfrak d\,\mathfrak n(\mathfrak n+1)\bigr)
        Y_{\mathfrak n}^{+\mathfrak m}
        \;-\;\mathfrak w\,\mathfrak m\,Y_{\mathfrak n}^{-\mathfrak m},\\
    \Lop Y_{\mathfrak n}^{-\mathfrak m}
      &=\bigl(\mathfrak r+\mathfrak d\,\mathfrak n(\mathfrak n+1)\bigr)
        Y_{\mathfrak n}^{-\mathfrak m}
        \;+\;\mathfrak w\,\mathfrak m\,Y_{\mathfrak n}^{+\mathfrak m} .
\end{align*}
Given a finite index set $\mathcal S$ of pairs $(\mathfrak n,\mathfrak m)$, we define the manufactured solution as
\[
    u\ :=\ \sum_{(\mathfrak n,\mathfrak m)\in\mathcal S}
           \mathfrak a_{\mathfrak n\mathfrak m}\,Y_{\mathfrak n}^{\mathfrak m},
\]
for coefficients $\mathfrak a_{\mathfrak n\mathfrak m}$.

\paragraph{Flat torus:}
In the case of the flat torus, we use that $\mathbb T^2 \cong \mathbb S^1 \times \mathbb S^1$, so that the coordinate system is $(\vartheta, \varphi) \in [0,2\pi)\times [0, 2\pi)$. The induced metric is $d\vartheta^{2}+d\varphi^{2}$, so the tangent vectors
\[
\hat e_{\vartheta}=(-\sin\vartheta,\cos\vartheta,0,0),\qquad
\hat e_{\varphi}=(0,0,-\sin\varphi,\cos\varphi),
\]
form an orthonormal frame of $\mathcal T_x\mathbb T^{2}$ at every point and the gradient reads
\[
\grad=\hat e_{\vartheta}\,\frac{\partial}{\partial\vartheta}
+\hat e_{\varphi}\,\frac{\partial}{\partial\varphi}.
\]
We take for $\mathfrak V$ the field with constant components in this frame,
\[
\mathfrak V\ :=\ \mathfrak w_1\,\hat e_{\vartheta}+\mathfrak w_2\,\hat e_{\varphi},
\qquad (\mathfrak w_1,\mathfrak w_2)\in\R^{2}\setminus\{0\}.
\]
The advection term becomes
\[
\mathfrak V\cdot\grad\ =\ \mathfrak w_1\,\frac{\partial}{\partial\vartheta}
+\mathfrak w_2\,\frac{\partial}{\partial\varphi} .
\]
The eigenpairs of $-\Delta$ on the flat torus are the standard Fourier modes, indexed by $(\mathfrak m,\mathfrak n)\in\mathbb Z^{2}$. In real form, we have
\[
    C_{\mathfrak m\mathfrak n}
      =\mathfrak c_{\mathfrak m\mathfrak n}
        \cos(\mathfrak m\vartheta+\mathfrak n\varphi),
    \qquad
    S_{\mathfrak m\mathfrak n}
      =\mathfrak c_{\mathfrak m\mathfrak n}
        \sin(\mathfrak m\vartheta+\mathfrak n\varphi),
\]
both with eigenvalue $\mathfrak m^{2}+\mathfrak n^{2}$, where $\mathfrak c_{\mathfrak m\mathfrak n}$ is the $L^{2}(\mathbb T^{2})$ normalization constant. With our choice of transport field, the advection term acts on a mode pair by
\begin{equation*}
    (\mathfrak V\cdot\grad)\,C_{\mathfrak m\mathfrak n}
      =-\,(\mathfrak w_1\mathfrak m+\mathfrak w_2\mathfrak n)\,
        S_{\mathfrak m\mathfrak n},
    \qquad
    (\mathfrak V\cdot\grad)\,S_{\mathfrak m\mathfrak n}
      =\,(\mathfrak w_1\mathfrak m+\mathfrak w_2\mathfrak n)\,
        C_{\mathfrak m\mathfrak n}.
\end{equation*}
Combining this with the eigenvalue relation gives the closed form
\begin{align*}
    \Lop C_{\mathfrak m\mathfrak n}
      &=\bigl(\mathfrak r+\mathfrak d\,
              (\mathfrak m^{2}+\mathfrak n^{2})\bigr)
        C_{\mathfrak m\mathfrak n}
        \;-\;(\mathfrak w_1\mathfrak m+\mathfrak w_2\mathfrak n)\,
        S_{\mathfrak m\mathfrak n},\\
    \Lop S_{\mathfrak m\mathfrak n}
      &=\bigl(\mathfrak r+\mathfrak d\,
              (\mathfrak m^{2}+\mathfrak n^{2})\bigr)
        S_{\mathfrak m\mathfrak n}
        \;+\;(\mathfrak w_1\mathfrak m+\mathfrak w_2\mathfrak n)\,
        C_{\mathfrak m\mathfrak n}.
\end{align*}
Given a finite index set $\mathcal S$ of pairs $(\mathfrak m,\mathfrak n)$, we define the manufactured solution as
\[
    u\ :=\ \sum_{(\mathfrak m,\mathfrak n)\in\mathcal S}
           \bigl(\mathfrak a^{c}_{\mathfrak m\mathfrak n}
                 C_{\mathfrak m\mathfrak n}
               +\mathfrak a^{s}_{\mathfrak m\mathfrak n}
                 S_{\mathfrak m\mathfrak n}\bigr),
\]
for coefficients $\mathfrak a^{c},\mathfrak a^{s}$.

\paragraph{Results on the circle:}

On the circle, we let $N$ range from $500$ to $32\,000$, for polynomial degrees $p\in\{2,3\}$,  and retain $100$ graph modes. We set $\mathcal S=\{1,2,3,5\}$ with $\mathfrak a_1=1.0$, $\mathfrak b_2=0.7$, $\mathfrak a_3=0.5$, $\mathfrak b_5=0.3$ and all other coefficients to zero, and the problem parameters are $\mathfrak r=0.01$, $\mathfrak d=0.1$, $\mathfrak w=1$.
 The number of collocation points is $5\%$ of $N$, with a minimum of $150$. The nodes are equispaced in $\vartheta$ and then perturbed by independent uniform displacements of up to $\pm5\%$ of the spacing. The radius $\varrho_N$ is selected as in Section~\ref{sec:implementation}, with the maximum defining $\hat\eta_N$ taken over $400$ random nodes.

The lower row of Figure~\ref{fig:circleconv} and the right hand side of Table~\ref{tab:circle} report the relative Gram matrix error $\max_{i,j}|\Khat_{ij}-K_{ij}|/\max_{i,j}|K_{ij}|$, which is the quantity bounded by Theorem~\ref{thm:main}. The graph kernel with $L^{\varrho_N}$ (\lsRhogr{}) converges, while the error of $L^{h}$ (\lsHgr{}) grows with $N$, as predicted by the amplification of $\eta_N$ by the stencil norm.
On this problem $\eta_N$ can be measured against the exact kernel. In Table~\ref{tab:circle} we report $\hat\eta_N/\eta_N$, which stays close to one for large enough $N$.
The exact kernel with $L^{h}$ (\lsHex{}) converges, consistently with Theorem~\ref{thm:known}, since there $\eta_N$ plays no role.

The upper row of Figure~\ref{fig:circleconv} and the left-hand side of Table~\ref{tab:circle} report the relative $L^2$ error against the exact solution at the nodes, which follows the same pattern. The direct GFDM solve (\lsGFDM{}) converges at the expected order at both degrees, and PMM with the exact operator and basis (\lsPMMex{}) remains at machine precision throughout. With the exact Mat\'ern kernel, $L^{h}$ follows the GFDM rate, since for that configuration $\eta_N$ plays no role. With the graph kernel, in contrast, $L^{h}$ does not converge. It is competitive on the coarsest clouds, then its error increases by an order of magnitude and decreases only slowly afterwards, remaining above $10^{-1}$ up to $N=32\,000$. With the same kernel, the same estimated frames and the same collocation set, $L^{\varrho_N}$ instead decreases monotonically from $N=1\,000$, reaching $1.74{\times}10^{-2}$ at $p=2$ and $4.84{\times}10^{-3}$ at $p=3$.

\begin{table}[t]
  \centering
  \caption{Advection--diffusion--reaction on $\mathbb S^{1}$. \textbf{Bold} marks the lower of the two errors in each metric at each cloud size. A dagger marks the clouds on which the admissibility window is empty. The reported value is the estimated ceiling and $\varrho_N$ is set to $1$ by the first fallback.}
  \label{tab:circle}
  \small
  \begin{tabular}{@{}rrrrrrrr@{}}
    \toprule
    & & & & \multicolumn{2}{c}{rel.\ $L^2$} & \multicolumn{2}{c}{$\max_{ij}|\Khat_{ij}-K_{ij}|\,/\,\max_{ij}|K_{ij}|$} \\
    \cmidrule(lr){5-6}\cmidrule(l){7-8}
    $N$ & $N_c$ & $\varrho_N$ & $\hat\eta_N/\eta_N$ & $L^{h}$ & $L^{\varrho_N}$ & $L^{h}$ & $L^{\varrho_N}$ \\
    \midrule
    \multicolumn{8}{@{}l}{\emph{Degree $p=2$}} \\
    $500$      & $150$     & $1.1107^{\dagger}$ & $2.63$ & $\mathbf{1.43{\times}10^{-1}}$ & $3.61$                         & $\mathbf{4.32{\times}10^{-2}}$ & $6.27{\times}10^{-1}$ \\
    $1\,000$   & $150$     & $0.6113$           & $1.86$ & $\mathbf{4.14{\times}10^{-2}}$ & $9.70{\times}10^{-1}$          & $\mathbf{8.26{\times}10^{-2}}$ & $2.70{\times}10^{-1}$ \\
    $2\,000$   & $150$     & $0.3064$           & $1.43$ & $\mathbf{9.99{\times}10^{-2}}$ & $1.10{\times}10^{-1}$          & $1.59{\times}10^{-1}$          & $\mathbf{7.68{\times}10^{-2}}$ \\
    $4\,000$   & $200$     & $0.1843$           & $1.19$ & $8.68{\times}10^{-1}$          & $\mathbf{3.52{\times}10^{-2}}$ & $3.64{\times}10^{-1}$          & $\mathbf{1.59{\times}10^{-2}}$ \\
    $8\,000$   & $400$     & $0.1519$           & $1.00$ & $4.14{\times}10^{-1}$          & $\mathbf{2.34{\times}10^{-2}}$ & $6.13{\times}10^{-1}$          & $\mathbf{1.24{\times}10^{-2}}$ \\
    $16\,000$  & $800$     & $0.1362$           & $0.87$ & $3.41{\times}10^{-1}$          & $\mathbf{1.88{\times}10^{-2}}$ & $1.82$                         & $\mathbf{1.28{\times}10^{-2}}$ \\
    $32\,000$  & $1\,600$  & $0.1314$           & $0.82$ & $2.26{\times}10^{-1}$          & $\mathbf{1.74{\times}10^{-2}}$ & $4.30$                         & $\mathbf{4.53{\times}10^{-3}}$ \\
    \addlinespace
    \multicolumn{8}{@{}l}{\emph{Degree $p=3$}} \\
    $500$      & $150$     & $1.0775^{\dagger}$ & $2.63$ & $\mathbf{1.36{\times}10^{-1}}$ & $3.44{\times}10^{-1}$          & $\mathbf{4.77{\times}10^{-2}}$ & $6.03{\times}10^{-2}$ \\
    $1\,000$   & $150$     & $1.0714^{\dagger}$ & $1.86$ & $\mathbf{6.17{\times}10^{-2}}$ & $3.52{\times}10^{-1}$          & $\mathbf{8.49{\times}10^{-2}}$ & $9.33{\times}10^{-2}$ \\
    $2\,000$   & $150$     & $0.6895$           & $1.43$ & $\mathbf{7.71{\times}10^{-2}}$ & $9.11{\times}10^{-2}$          & $1.60{\times}10^{-1}$          & $\mathbf{2.74{\times}10^{-2}}$ \\
    $4\,000$   & $200$     & $0.3457$           & $1.19$ & $8.84{\times}10^{-1}$          & $\mathbf{2.12{\times}10^{-2}}$ & $3.71{\times}10^{-1}$          & $\mathbf{1.03{\times}10^{-3}}$ \\
    $8\,000$   & $400$     & $0.1853$           & $1.00$ & $4.38{\times}10^{-1}$          & $\mathbf{6.19{\times}10^{-3}}$ & $6.13{\times}10^{-1}$          & $\mathbf{4.17{\times}10^{-4}}$ \\
    $16\,000$  & $800$     & $0.1692$           & $0.87$ & $3.37{\times}10^{-1}$          & $\mathbf{5.13{\times}10^{-3}}$ & $1.83$                         & $\mathbf{3.78{\times}10^{-4}}$ \\
    $32\,000$  & $1\,600$  & $0.1642$           & $0.82$ & $2.24{\times}10^{-1}$          & $\mathbf{4.84{\times}10^{-3}}$ & $4.25$                         & $\mathbf{2.44{\times}10^{-4}}$ \\
    \bottomrule
  \end{tabular}
\end{table}

\begin{figure}[htbp]
    \centering
    \includegraphics[width=1\textwidth]{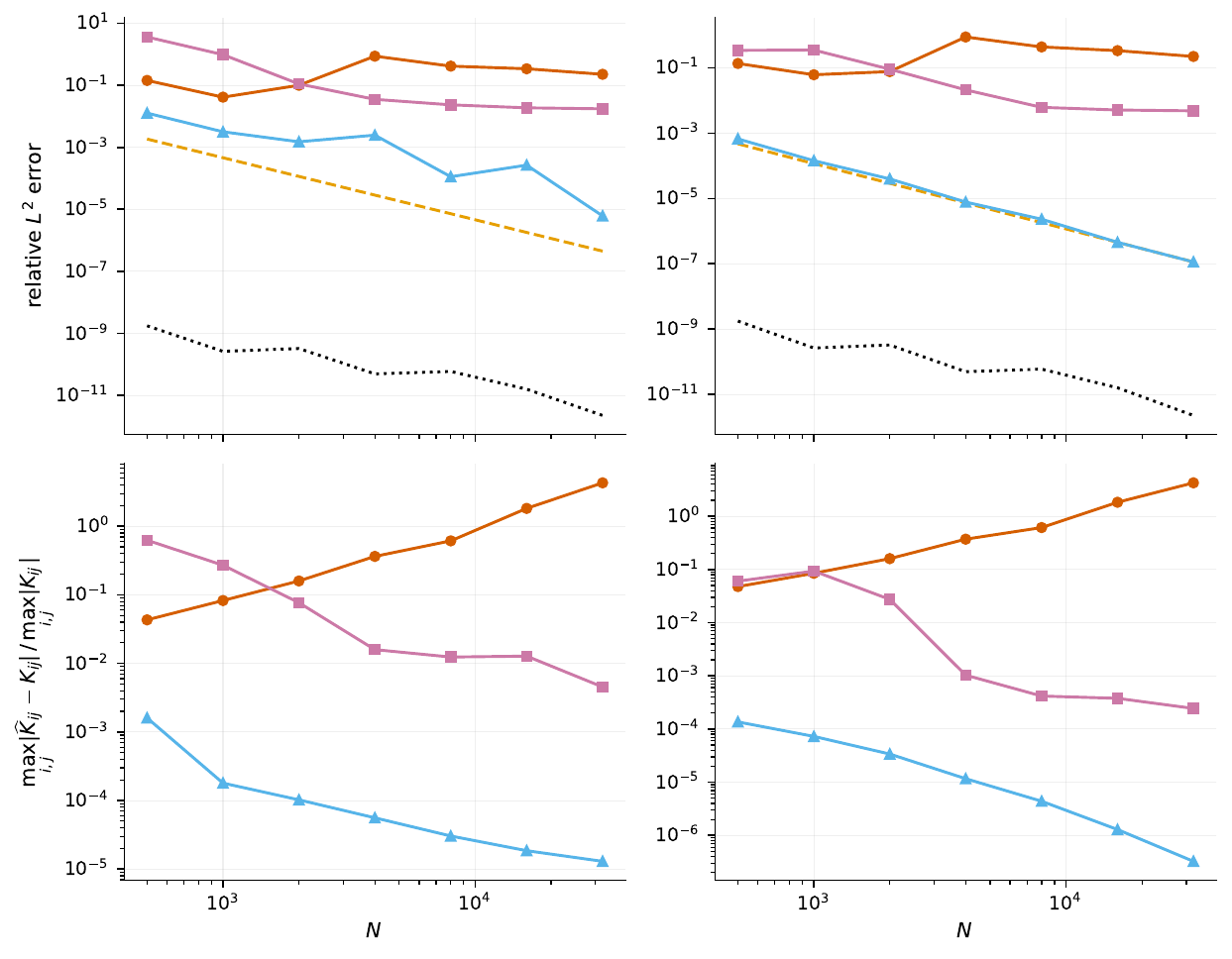}
    \caption{Advection--diffusion--reaction on $\mathbb S^{1}$. Left column $p=2$ and right column $p=3$. \emph{Top:} relative $L^2$ error against the exact solution. \emph{Bottom:} $\max_{i,j}|\Khat_{ij}-K_{ij}|$. The curves are the direct GFDM solve (\lsGFDM{}), PMM with exact operator and basis (\lsPMMex{}), PMM with $L^{h}$ and the exact Mat\'ern kernel (\lsHex{}), PMM with $L^{h}$ and the graph kernel (\lsHgr{}) and PMM with $L^{\varrho_N}$ and the graph kernel (\lsRhogr{}).}
    \label{fig:circleconv}
\end{figure}

\paragraph{Results on the sphere and torus:}

The problem parameters are set to $\mathfrak r=1$ and $\mathfrak d=0.1$. On the sphere, the transport is the rotation of angular velocity $\mathfrak w=1$ about the polar axis $\mathfrak e$, while on the flat torus $(\mathfrak w_1,\mathfrak w_2)=(1,\,0.35)$. In both cases, the point cloud $X_N$ is given in ambient coordinates, in $\R^{3}$ for the sphere and $\R^{4}$ for the flat torus, with $N=10^{4}$ and $N_c=250$ collocation points. Both discrete operators use polynomial degree $p=2$, and $L^{h}$ is built on $\mathsf K=25$ nearest neighbours. In the sphere case, we take $\mathcal S=\{(1,0),(2,1),(3,-2),(4,1)\}$ in the $(\mathfrak n,\mathfrak m)$ indexing, while for the flat torus $\mathcal S=\{(1,0),(1,1),(2,1),(2,2)\}$ with coefficients $\{1.0,0.7,0.5,0.3\}$.

Table~\ref{tab:advdiff-steady} reports, for both manifolds, the relative $L^2$ and $L^\infty$ error,  the empirical coverage of the credible intervals and the median posterior standard deviation. PMM with exact operator and kernel achieves by far the lowest error rates, and its credible intervals contain the true solution at every node at both nominal levels, while the posterior standard deviation remains small.
Replacing the exact Mat\'ern kernel by the graph kernel, at fixed operator and fixed collocation set, increases the error by roughly one order of magnitude.
The direct GFDM solve and $L^{\varrho_N}$ both perform poorly for different reasons. The former inverts $L^{h}$, whose condition number degrades as $N$ grows. For the latter, the cloud is too small for the admissible range of Section~\ref{sec:implementation} to be nonempty, so that $\varrho_N$ is set to $1$ by the first fallback, and the resulting credible intervals are strongly overconfident.
Figures~\ref{fig:spherefig} and~\ref{fig:torusfig} show the solution returned by each configuration in the top row, and the corresponding absolute error against the manufactured solution in the bottom row. The failure of $L^{\varrho_N}$ appears as a smooth, large-scale deviation rather than as local noise, while the direct GFDM solve degrades on the whole domain.

\begin{table}[t]
  \centering
  \caption{Steady advection--diffusion--reaction on $\mathbb S^2$ and $\mathbb{T}^2$. The first two columns are the relative $L^2$ and $L^\infty$ error against the ground truth. The remaining columns are the empirical coverage of the credible intervals at those nominal levels and the median posterior standard deviation, and exist only for the probabilistic methods. \textbf{Bold} marks the lowest error in each column within each block.}
  \label{tab:advdiff-steady}
  \resizebox{\linewidth}{!}{%
    \begin{tabular}{@{}lrrrrr@{}}
      \toprule
      \multicolumn{6}{c}{\textbf{Sphere $\mathbb S^2$}} \\
      Method & rel.\ $L^2$ & rel.\ $L^\infty$ & cov.\ $68.3\%$ (\%) & cov.\ $95\%$ (\%) & med.\ post.\ sd \\
      \midrule
      \multicolumn{6}{@{}l}{\emph{Known manifold, analytic tangent frames}} \\
      GFDM $L^h$, direct solve & $4.61{\times}10^{-1}$ & $2.17$ & -- & -- & -- \\
      PMM, exact operator and basis (ours) & $\mathbf{1.38{\times}10^{-8}}$ & $\mathbf{3.30{\times}10^{-8}}$ & $100.0$ & $100.0$ & $1.97{\times}10^{-7}$ \\
      PMM, $L^h$ + exact Mat\'ern (ours) & $5.48{\times}10^{-3}$ & $5.95{\times}10^{-3}$ & $51.2$ & $72.8$ & $3.84{\times}10^{-5}$ \\
      \addlinespace
      \multicolumn{6}{@{}l}{\emph{Implicit manifold, estimated tangent frames}} \\
      PMM, $L^h$ + graph Mat\'ern (ours) & $\mathbf{4.82{\times}10^{-2}}$ & $\mathbf{5.91{\times}10^{-2}}$ & $69.2$ & $91.2$ & $5.59{\times}10^{-4}$ \\
      PMM, $L^{\varrho_N}$ + graph Mat\'ern (ours) & $4.43{\times}10^{-1}$ & $4.60{\times}10^{-1}$ & $18.4$ & $33.9$ & $1.39{\times}10^{-3}$ \\
      \midrule
      \multicolumn{6}{c}{\textbf{Flat torus $\mathbb{T}^2$}} \\
      Method & rel.\ $L^2$ & rel.\ $L^\infty$ & cov.\ $68.3\%$ (\%) & cov.\ $95\%$ (\%) & med.\ post.\ sd \\
      \midrule
      \multicolumn{6}{@{}l}{\emph{Known manifold, analytic tangent frames}} \\
      GFDM $L^h$, direct solve & $7.60{\times}10^{-2}$ & $2.90{\times}10^{-1}$ & -- & -- & -- \\
      PMM, exact operator and basis (ours) & $\mathbf{1.96{\times}10^{-8}}$ & $\mathbf{4.05{\times}10^{-8}}$ & $100.0$ & $100.0$ & $4.28{\times}10^{-6}$ \\
      PMM, $L^h$ + exact Mat\'ern (ours) & $5.69{\times}10^{-3}$ & $8.27{\times}10^{-3}$ & $38.6$ & $59.1$ & $5.05{\times}10^{-4}$ \\
      \addlinespace
      \multicolumn{6}{@{}l}{\emph{Implicit manifold, estimated tangent frames}} \\
      PMM, $L^h$ + graph Mat\'ern (ours) & $\mathbf{3.49{\times}10^{-2}}$ & $\mathbf{5.92{\times}10^{-2}}$ & $79.5$ & $95.1$ & $8.41{\times}10^{-3}$ \\
      PMM, $L^{\varrho_N}$ + graph Mat\'ern (ours) & $2.80{\times}10^{-1}$ & $3.52{\times}10^{-1}$ & $27.0$ & $42.5$ & $1.65{\times}10^{-2}$ \\
      \bottomrule
    \end{tabular}%
  }
\end{table}

\begin{figure}[htbp]
    \centering
    \noindent\pcell{6}{(a)}\pcell{6}{(b)}\pcell{6}{(c)}\pcell{6}{(d)}\pcell{6}{(e)}\pcell{6}{(f)}\par\nointerlineskip
    \includegraphics[width=1\textwidth]{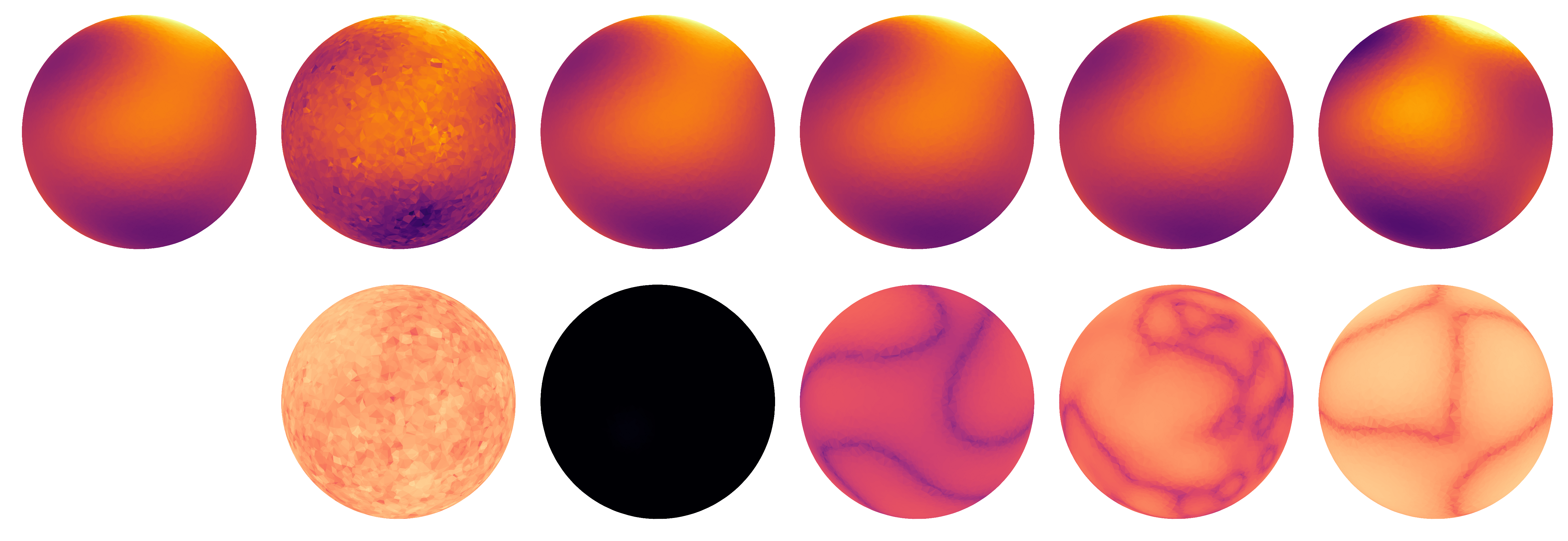}
    \caption{Steady advection--diffusion--reaction on $\mathbb S^{2}$. The top row reports the predictions for (a) the manufactured exact solution, (b) the direct GFDM solve with $L^{h}$, (c) PMM with the exact operator and basis, (d) PMM with $L^{h}$ and the exact Mat\'ern kernel, (e) PMM with $L^{h}$ and the graph kernel, and (f) PMM with $L^{\varrho_N}$ and the graph kernel. The bottom row reports the absolute error against the manufactured solution, that is $|u-\tilde\mu|$ at every node, of the panel directly above it; the cell under (a) is empty, the exact solution having no error against itself. Panels (a)--(d) use the analytic tangent frames, (e) and (f) the PCA estimates.}
    \label{fig:spherefig}
\end{figure}

\begin{figure}[htbp]
    \centering
    \noindent\pcell{6}{(a)}\pcell{6}{(b)}\pcell{6}{(c)}\pcell{6}{(d)}\pcell{6}{(e)}\pcell{6}{(f)}\par\nointerlineskip
    \includegraphics[width=1\textwidth]{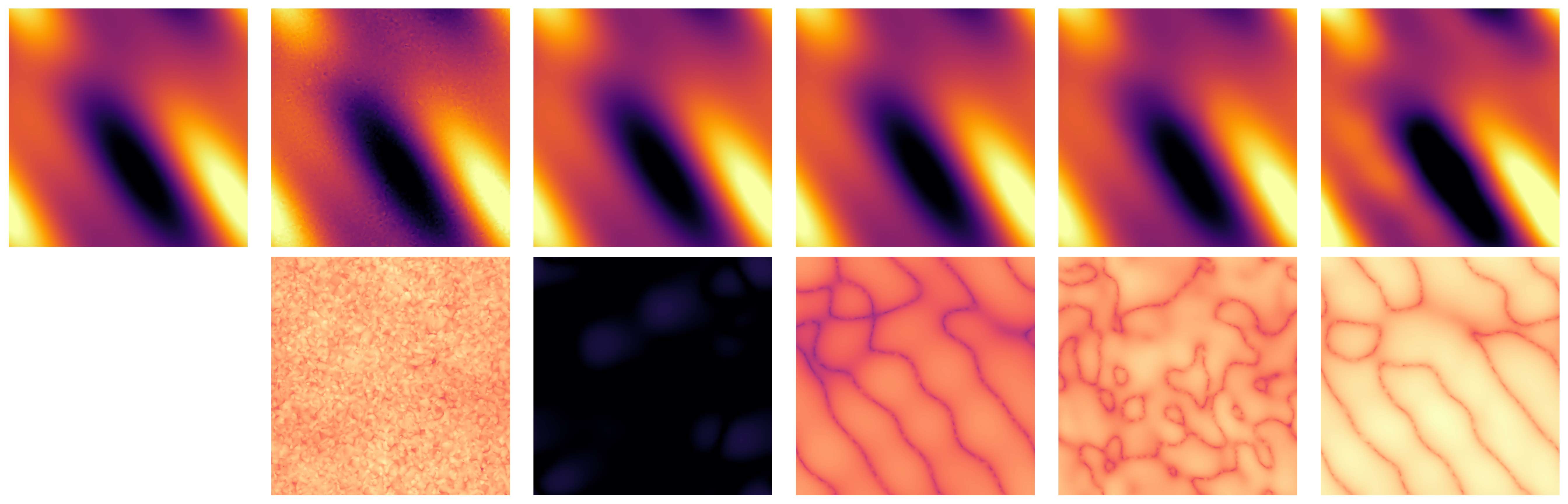}
    \caption{Steady advection--diffusion--reaction on $\mathbb{T}^{2}$, with the panels drawn in the fundamental domain $(\vartheta,\varphi)\in[0,2\pi)^{2}$. The top row reports the predictions for (a) the manufactured exact solution, (b) the direct GFDM solve with $L^{h}$, (c) PMM with the exact operator and basis, (d) PMM with $L^{h}$ and the exact Mat\'ern kernel, (e) PMM with $L^{h}$ and the graph kernel, and (f) PMM with $L^{\varrho_N}$ and the graph kernel. The bottom row reports the absolute error against the manufactured solution of the panel directly above it; the cell under (a) is empty, the exact solution having no error against itself. Panels (a)--(d) use the analytic tangent frames, (e) and (f) the PCA estimates.}
    \label{fig:torusfig}
\end{figure}
\subsection{Brain cortical thickness smoothing}

The cerebral cortex is a thin layer of neural tissue in the outermost part of the brain, where most of the neuronal activity takes place. From a geometric point of view, it is modeled as a closed surface embedded in $\R^3$. The thickness of this surface is usually referred to as cortical thickness and can be estimated from magnetic resonance images (MRI) through several steps of image processing, described for example in~\citet{chung2005}. This procedure introduces noise, which is particularly undesirable in group-level analysis, so that the measurements are usually smoothed. One of the most widely used approaches in this context is kernel smoothing, which convolves the measurements with a heat kernel whose time parameter determines the amount of smoothing. Denoting the cortical surface with $\Mfd$ and the noisy measurements by $g(x)$, heat kernel smoothing is equivalent to solving the heat equation~\citep{chung2005}
\[
    \left\{
    \begin{aligned}
        &\partial_tu - \Delta u=0 && \text{in }\Mfd \times (0,T],\\
        &u(\cdot, 0)=g && \text{in }\Mfd,
    \end{aligned}
    \right.
\]
whose solution at time $t$ is given by the heat semigroup, $u(\cdot, t)=e^{t\Delta}g$. Since the manifold is represented by a mesh, the problem must be discretized, and several approaches have been proposed in the literature. \citet{chung2005} propose multiple applications of the heat kernel convolution against initial data, where the kernel is locally approximated with the exponential of the distance between nodes. ~\citet{chung2004diffusion} instead solve the heat equation with finite elements in space and $n$ explicit Euler steps in time. More recently, ~\citet{seo2010} instead focuses on the estimation of the Laplacian eigendecomposition, so that the convolution can be represented via the spectral Fourier series. In practice, the eigenpairs are obtained using cotangent discretization~\citep{pinkall1993computing}, and the coefficients are estimated by least squares. This approach is exact in time and it has strong convergence guarantees in space. Nevertheless, it requires a triangulation of the domain and a considerable computational effort. Since small values of $t$ are of most interest in practice, to avoid over-smoothing the data, we propose a simpler alternative, namely solving a single backward Euler step for the heat equation
\begin{equation}\label{eq:onestep}
    (I-t\Delta)u = g,
\end{equation}
where $t$ plays the role of the smoothing parameter. Since~\eqref{eq:onestep} is a linear elliptic PDE on a closed manifold, it falls within our framework of PMMs on manifolds.

As a reference, we take the solution of~\citet{seo2010} with Laplace--Beltrami eigenpairs obtained by the cotangent discretization, retaining $1000$ modes. At each smoothing time $t$, we compare the two established discretizations, namely the iterated kernel smoothing of~\citet{chung2005} and the finite element diffusion of~\citet{chung2004diffusion} with $n$ time steps, and our approach, which solves~\eqref{eq:onestep} with PMMs using either $L^h$ or $L^{\varrho_N}$. We stress that both competing methods use the measurements at every vertex, while PMMs condition on a limited number of collocation points and return the smoothed field, together with an uncertainty estimate, at every vertex.

All compared methods are implemented in Python following the original papers and the MATLAB routines distributed by~\citet{chung2005}, namely \texttt{hk\_smooth} for the iterated kernel smoothing, \texttt{diffusion\_smooth} for the finite-element diffusion of~\citet{chung2004diffusion} and \texttt{lb\_smooth} for the eigenfunction expansion of~\citet{seo2010}. The surface is the inflated left hemisphere of the \texttt{fsaverage5} template, with $10\,242$ vertices.

\paragraph{Results.}

Table~\ref{tab:cortical-accuracy} reports the comparison at three smoothing times. While both competing methods use the measured thickness at all vertices, the PMMs impose~\eqref{eq:onestep} at $N_c=799$ collocation points only, that is, $7.8\%$ of the vertices, and use a single step instead of the $n$ steps of~\citet{chung2004diffusion}. Both PMMs use polynomial degree $p=3$. Under this budget, PMMs improve on the iterated kernel smoothing of~\citet{chung2005} at every $t$, and remain within a small factor of the finite-element diffusion of~\citet{chung2004diffusion}, both in relative $L^2$ and $L^\infty$. Our approach thus attains a comparable accuracy from less than a tenth of the observations, predicting the smoothed field at the remaining $92.2\%$ of the vertices and providing a credible interval at every vertex, which neither competing method does.

The empirical coverage of both PMMs is close to the nominal levels at $t=0.002$ and degrades as $t$ increases. The two stencils give nearly identical results on this surface. Indeed, the estimated reach is smaller than the fill distance, so that the admissible range is empty and $\varrho_N$ is set by the second fallback, at a value close to the $L^{h}$ stencil.
Figure~\ref{fig:cortical_pred} shows the smoothed field at three diffusion times, with the measured thickness and the reference in the first two columns, and Figure~\ref{fig:cortical_err} shows the absolute error against the reference. Since all panels share the same colour scale, the increasing smoothing of the field with $t$ is visible across rows.

\begin{table}[t]
  \centering
  \caption{Cortical thickness smoothing at three diffusion times $t$. Errors are relative $L^2$ and $L^\infty$ against the heat semigroup $e^{t\Delta}g$, computed as the eigenfunction expansion of~\citet{seo2010} in the cotangent basis truncated at $1\,000$ modes. The last three columns are the empirical coverage of the credible intervals at the two nominal levels and the median posterior standard deviation. \textbf{Bold} marks the lowest error in each column within each block.}
  \label{tab:cortical-accuracy}
    \resizebox{\linewidth}{!}{%
  \begin{tabular}{lrrrrr}
    \toprule
    Method & rel.\ $L^2$ & rel.\ $L^\infty$ & cov.\ $68.3\%$ (\%) & cov.\ $95\%$ (\%) & med.\ post.\ sd \\
    \midrule
    \multicolumn{6}{l}{\emph{$t = 0.002$}} \\
    PMM, $L^h$ + graph Mat\'ern (ours) & $3.50{\times}10^{-2}$ & $1.08{\times}10^{-1}$ & $64.5$ & $91.0$ & $0.067$ \\
    PMM, $L^{\varrho_N}$ + graph Mat\'ern (ours) & $3.62{\times}10^{-2}$ & $1.09{\times}10^{-1}$ & $64.0$ & $90.9$ & $0.069$ \\
    Iterated kernel smoothing \citep{chung2005} & $6.52{\times}10^{-2}$ & $4.16{\times}10^{-1}$ & -- & -- & -- \\
    Explicit FEM diffusion \citep{chung2004diffusion} & $\mathbf{2.21{\times}10^{-2}}$ & $\mathbf{7.84{\times}10^{-2}}$ & -- & -- & -- \\
    \addlinespace
    \multicolumn{6}{l}{\emph{$t = 0.01$}} \\
    PMM, $L^h$ + graph Mat\'ern (ours) & $3.38{\times}10^{-2}$ & $\mathbf{1.08{\times}10^{-1}}$ & $43.6$ & $69.0$ & $0.034$ \\
    PMM, $L^{\varrho_N}$ + graph Mat\'ern (ours) & $3.40{\times}10^{-2}$ & $1.10{\times}10^{-1}$ & $44.1$ & $70.0$ & $0.035$ \\
    Iterated kernel smoothing \citep{chung2005} & $5.41{\times}10^{-2}$ & $1.83{\times}10^{-1}$ & -- & -- & -- \\
    Explicit FEM diffusion \citep{chung2004diffusion} & $\mathbf{3.01{\times}10^{-2}}$ & $1.28{\times}10^{-1}$ & -- & -- & -- \\
    \addlinespace
    \multicolumn{6}{l}{\emph{$t = 0.05$}} \\
    PMM, $L^h$ + graph Mat\'ern (ours) & $3.82{\times}10^{-2}$ & $1.68{\times}10^{-1}$ & $26.9$ & $51.3$ & $0.018$ \\
    PMM, $L^{\varrho_N}$ + graph Mat\'ern (ours) & $3.81{\times}10^{-2}$ & $1.71{\times}10^{-1}$ & $29.8$ & $54.9$ & $0.018$ \\
    Iterated kernel smoothing \citep{chung2005} & $1.27{\times}10^{-1}$ & $4.42{\times}10^{-1}$ & -- & -- & -- \\
    Explicit FEM diffusion \citep{chung2004diffusion} & $\mathbf{2.39{\times}10^{-2}}$ & $\mathbf{6.85{\times}10^{-2}}$ & -- & -- & -- \\
    \bottomrule
  \end{tabular}
  }
\end{table}

\begin{figure}[htbp]
    \centering
    \sbox{\rowbox}{\includegraphics[width=0.96\textwidth]{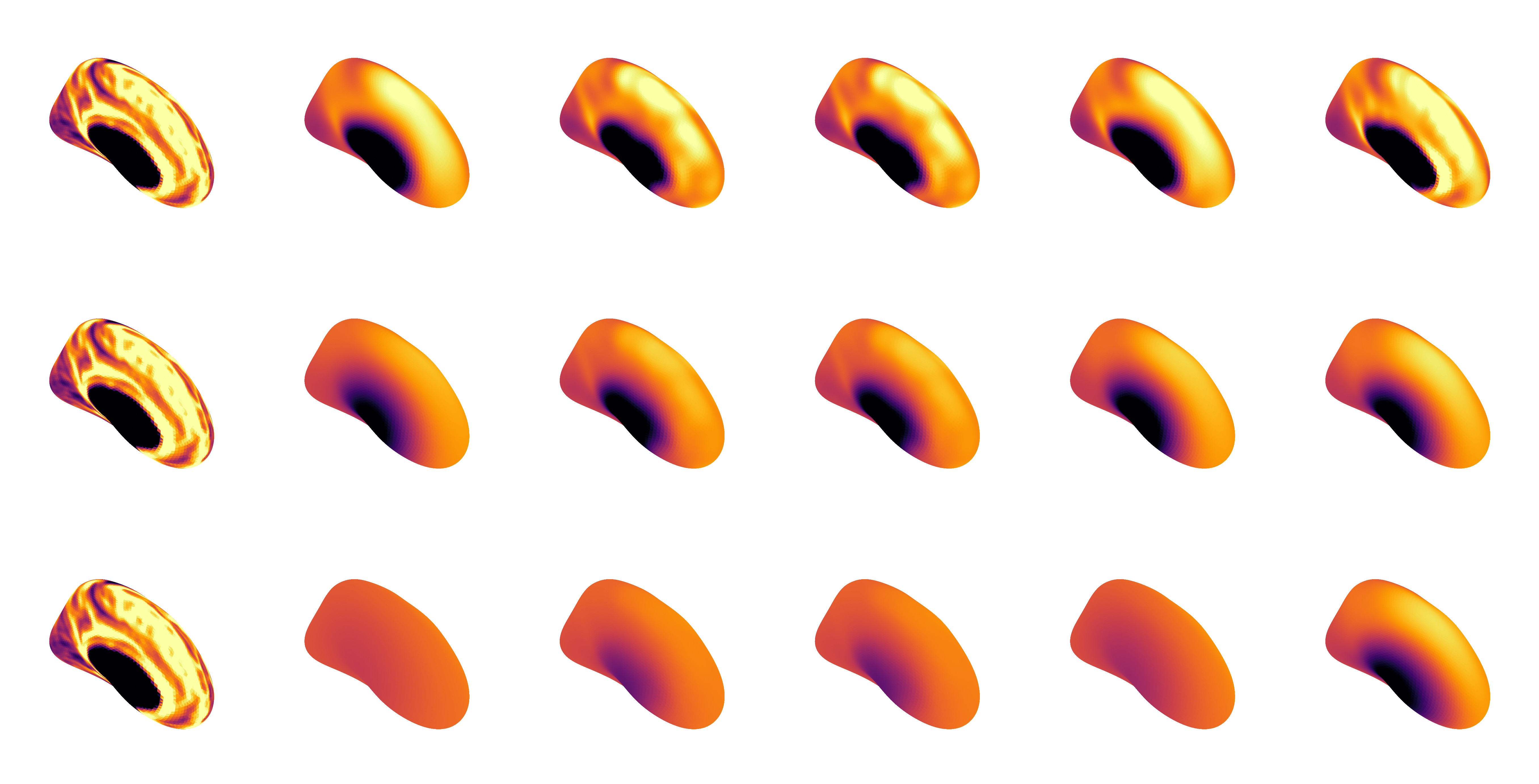}}%
    \settoheight{\rowht}{\usebox{\rowbox}}%
    \noindent
    \begin{minipage}[b]{0.03\textwidth}%
      \vbox to \rowht{%
        \vskip 0pt plus 1fil
        \hbox to \linewidth{\hfil\rotatebox{90}{\footnotesize $t=0.002$}\hfil}
        \vskip 0pt plus 2fil
        \hbox to \linewidth{\hfil\rotatebox{90}{\footnotesize $t=0.01$}\hfil}
        \vskip 0pt plus 2fil
        \hbox to \linewidth{\hfil\rotatebox{90}{\footnotesize $t=0.05$}\hfil}
        \vskip 0pt plus 1fil}%
    \end{minipage}%
    \begin{minipage}[b]{0.96\textwidth}%
      \noindent\pcell{6}{(a)}\pcell{6}{(b)}\pcell{6}{(c)}\pcell{6}{(d)}\pcell{6}{(e)}\pcell{6}{(f)}\par\nointerlineskip
      \usebox{\rowbox}%
    \end{minipage}
    \caption{Cortical thickness smoothing on the inflated left hemisphere of \texttt{fsaverage5}, $10\,242$ vertices drawn on the triangulation for (a) the measured thickness $g$, which is the datum being smoothed and is therefore identical in every row; (b) the reference, the eigenfunction expansion of~\citet{seo2010} in the cotangent basis at $1000$ modes; (c) PMM with $L^{h}$; (d) PMM with $L^{\varrho_N}$; (e) the finite-element diffusion of~\citet{chung2004diffusion}; (f) the iterated kernel smoothing of~\citet{chung2005}. Columns (c) and (d) condition on $N_c=799$ vertices, $7.8\%$ of the mesh, while (b), (e) and (f) read all of them.}
    \label{fig:cortical_pred}
\end{figure}

\begin{figure}[htbp]
    \centering
    \sbox{\rowbox}{\includegraphics[width=0.96\textwidth]{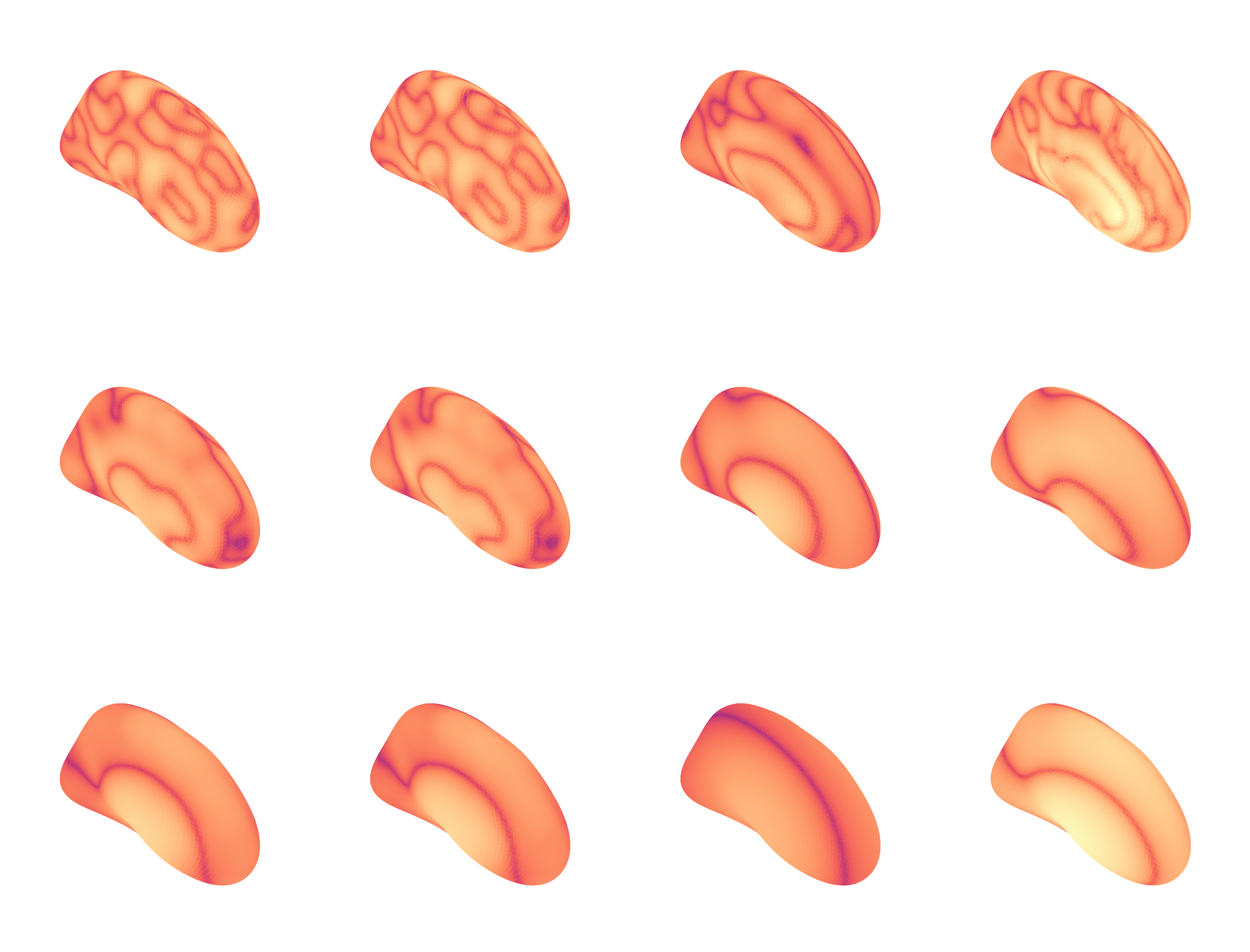}}%
    \settoheight{\rowht}{\usebox{\rowbox}}%
    \noindent
    \begin{minipage}[b]{0.03\textwidth}%
      \vbox to \rowht{%
        \vskip 0pt plus 1fil
        \hbox to \linewidth{\hfil\rotatebox{90}{\footnotesize $t=0.002$}\hfil}
        \vskip 0pt plus 2fil
        \hbox to \linewidth{\hfil\rotatebox{90}{\footnotesize $t=0.01$}\hfil}
        \vskip 0pt plus 2fil
        \hbox to \linewidth{\hfil\rotatebox{90}{\footnotesize $t=0.05$}\hfil}
        \vskip 0pt plus 1fil}%
    \end{minipage}%
    \begin{minipage}[b]{0.96\textwidth}%
      \noindent\pcell{4}{(a)}\pcell{4}{(b)}\pcell{4}{(c)}\pcell{4}{(d)}\par\nointerlineskip
      \usebox{\rowbox}%
    \end{minipage}
    \caption{Absolute error of each compared method against the reference of Figure~\ref{fig:cortical_pred}(b) for (a) PMM with $L^{h}$; (b) PMM with $L^{\varrho_N}$; (c) the finite-element diffusion of~\citet{chung2004diffusion}; (d) the iterated kernel smoothing of~\citet{chung2005}.}
    \label{fig:cortical_err}
\end{figure}

\FloatBarrier
\section{Conclusions}
This paper extended probabilistic meshless methods to manifold domains and quantified their convergence according to the information available about the manifold. When the manifold and the eigendecomposition of its Laplace--Beltrami operator are known, the construction carries over directly from the Euclidean setting. If, in addition, the differential operator is not available in closed form, it is replaced by a GFDM discretization, which is consistent at the nodes by Theorem~\ref{thm:known}. When the manifold is implicit, both the kernel and the operator are estimated from the point cloud, and we proposed a novel GFDM $L^{\varrho_N}$ which, by Theorem~\ref{thm:main}, yields uniform convergence of the Gram matrix. Finally, Theorem~\ref{thm:contraction} bounds the error of the posterior mean, separating the collocation error intrinsic to PMMs from the error due to the approximation of the operator and, in the implicit case, of the kernel. This bound reduces to the estimate of~\citet[Proposition~6]{cockayne2017probabilistic} when no approximation is made.

The experiments support our analysis. On the circle, the Gram matrix obtained with the standard GFDM stencil drifts away from its target as $N$ grows, while the one obtained by $L^{\varrho_N}$ converges, and the solution errors follow the same pattern. On the sphere and on the flat torus, the clouds are too coarse for the admissible range of the stencil radius to be nonempty and the standard stencil remains preferable. This is consistent with the theory, since $L^{\varrho_N}$ is designed for the asymptotic regime in which the amplification of the kernel error dominates. On cortical thickness data, our approach attains an accuracy comparable to that of established discretizations, while conditioning on fewer than a tenth of the vertices and providing a credible interval at every vertex.

Our approach has two main limitations. First, the graph kernel is only defined on the point cloud, so that, in the implicit case, the posterior is a distribution over the values at the nodes rather than a Gaussian process on the manifold, and convergence is only established at the nodes. Second, the admissible range of the stencil radius is nonempty only for large $N$, a regime in which computing both the graph Laplacian and the stencils becomes expensive.

\bibliography{ref}
\appendix

\section{The graph Laplacian and its spectral convergence}\label{app:graph}
Suppose $\Mfd$ is a $d$-dimensional closed manifold embedded in a higher dimensional Euclidean space. Consider a Gaussian kernel on the ambient space, $k_\zeta(x,x')=\exp\bigl(-\|x-x'\|_2^2/4\zeta^2\bigr)$ with $\zeta>0$, and denote with $\mathsf d_\zeta(x_i)=\sum_{j=1}^{N}k_\zeta(x_i,x_j)$ the associated degree. Define the affinity matrix as
\begin{equation}
    W^{(\alpha)}_{ij}:=\frac{k_\zeta(x_i,x_j)}
       {\mathsf d_\zeta(x_i)^{\alpha}\mathsf d_\zeta(x_j)^{\alpha}},
    \qquad \alpha \in [0,1],
\end{equation}
the degree matrix $D^{(\alpha)}_{ii}=\sum_j W^{(\alpha)}_{ij}$ and the transition matrix $\mathsf A^{(\alpha)}=\bigl(D^{(\alpha)}\bigr)^{-1}W^{(\alpha)}$. The graph Laplacian, which approximates $\Delta_\Mfd$, is defined as~\citep{Dunson2021}
\begin{equation}
    \mathsf L^{(\alpha)}=\frac{\mathsf A^{(\alpha)}-I}{\zeta^{2}} .
\end{equation}
For $\alpha=0$ this is the random walk graph Laplacian built on the unnormalized affinity matrix. Denote the spectrum of $-\mathsf L^{(\alpha)}$ by $0=\hat\lambda_{0}\le\hat\lambda_{1}\le\dots\le\hat\lambda_{N-1}\le 2\zeta^{-2}$, with eigenvectors $\hat\phi_{0},\dots,\hat\phi_{N-1}$ rescaled as in \citet[Definition~1]{Dunson2021}, so that they approximate the $L^{2}(\Mfd)$-normalized eigenfunctions. Assume that the eigenvalues of $-\Delta_\Mfd$ are simple, that the sampling density is smooth and non-vanishing, and that $\zeta=\zeta(N)$ is chosen properly. Then, for $N$ large and for a suitable sign $\varsigma_r\in\{-1,1\}$, \citet{Dunson2021} give, in the case $\alpha=1$,
\begin{equation}\label{eq:dww1}
    \bigl|\hat\lambda_{r}-\lambda^{\Delta}_r\bigr|\le C\Bigl(\tfrac{\log N}{N}\Bigr)^{\frac{3}{8d+26}},
    \qquad
    \max_{x_j\in X_N}\bigl|\varsigma_r\hat\phi_{r}^{(j)}-\phi_r(x_j)\bigr|
      \le C\Bigl(\tfrac{\log N}{N}\Bigr)^{\frac{1}{8d+16}},
\end{equation}
while for $\alpha=0$ under uniform sampling
\begin{equation}\label{eq:dww0}
    \bigl|\hat\lambda_{r}-\lambda^{\Delta}_r\bigr|\le C\Bigl(\tfrac{\log N}{N}\Bigr)^{\frac{1}{d+6}},
    \qquad
    \max_{x_j\in X_N}\bigl|\varsigma_r\hat\phi_{r}^{(j)}-\phi_r(x_j)\bigr|
      \le C\Bigl(\tfrac{\log N}{N}\Bigr)^{\frac{1}{d+4}} .
\end{equation}

\section{Convergence on implicit manifolds}\label{app:implicit}
Let $\Mfd$ be a smooth Riemannian $d$-manifold with geodesic distance $d_g(\cdot,\cdot)$ and volume $\dvol$. Throughout this appendix, $B$ denotes a geodesic ball in $\Mfd$ and $\mathsf B$ a Euclidean ball in $\R^{d}$. Write $\exp_x:\mathcal T_x\Mfd\to\Mfd$ for the exponential map of the Levi-Civita connection. Recalling the definitions of Section~\ref{sec:implementation}, the injectivity radius at $x$, $\inj(x)$, is the largest radius such that $\exp_x$ is a diffeomorphism from the ball of that radius onto its image, and it is easy to see that $\inj:=\inf_{x\in\Mfd}\inj(x)>0$ on a compact manifold. Fixing an orthonormal basis of $\mathcal T_{x_i}\Mfd$ and identifying $\mathcal T_{x_i}\Mfd\cong\R^{d}$, we write
\[
   \xi_i:=\exp_{x_i}^{-1}\;:\;B(x_i,\inj)\longrightarrow \mathsf B(0,\inj)\subset\R^{d}
\]
for the normal chart at $x_i$, so that $\xi_i(x_i)=0$ and
\begin{equation}\label{eq:radial}
   |\xi_i(y)|=d_g(x_i,y),
\end{equation}
for all $y\in B(x_i,\inj)$. Moreover, as in Section~\ref{sec:implementation}, the normal chart distorts geodesic distances by a bounded factor
\begin{equation}\label{eq:chartdistortion}
C_\inj^{-1}d_g(y,y’)\ \le\ |\xi_i(y)-\xi_i(y’)|\ \le\ C_\inj d_g(y,y’),
\end{equation}
for all $i$ and all $y,y'\in B(x_i,\inj/2)$. It is finite on a compact manifold and depends only on $\Mfd$, and $C_\inj=1$ exactly when $\Mfd$ is flat.
It is convenient to regard functions on $\Mfd$ as differentiated through this atlas. In other words, $u$ locally takes the form $u\circ\exp_{x}$ and it is a function of $z\in\R^{d}$. A linear differential operator of order $l$ acts as
\begin{equation}\label{eq:Linchart}
   (\Lop u)(x)=\sum_{|\beta|\le l}a_\beta(x)\,
     \partial^{\beta}_{z}\bigl(u\circ\exp_{x}\bigr)(0),
\end{equation}
with $A_\infty$ and $C_{\Lop}$ as in Assumption~\ref{ass:op}. We write
\[
   \|w\|_{C^{s}(\Mfd)}:=\sup_{x\in\Mfd}\ \sup_{|z| \leq \inj/2} \max_{|\gamma|\le s}\
   \bigl|\partial^{\gamma}_{z}\bigl(w\circ\exp_x\bigr)(z)\bigr| ,
\]
and for a kernel the mixed norm
\begin{equation}\label{eq:mixednorm}
   \|k\|_{C^{p+1,p+1}}
   :=\max_{|\gamma|\le p+1}\max_{|\delta|\le p+1}\ \sup_{x,y\in\Mfd} \sup_{|z_x|, |z_y| \leq \inj/2}
     \bigl|\partial^{\gamma}_{z_x}\partial^{\delta}_{z_y}
 \bigl[k\bigl(\exp_x(z_x),\exp_y(z_y)\bigr)\bigr]\bigr| .
\end{equation}
For multi-indices $\beta\in\N_0^{d}$ we use the notation $|\beta|=\sum_s\beta_s$, $\beta!=\prod_s\beta_s!$, $z^{\beta}=\prod_s z_s^{\beta_s}$, and $\Pip=\spn\{z^{\beta}:|\beta|\le p\}$, of dimension $Q=\binom{p+d}{d}$. Finally, $\Rest:C(\Mfd)\to\R^{N}$ is the restriction operator which acts as $(\Rest f)_i=f(x_i)$ for all $x_i \in X_N$. We also define norms for finite dimensional vectors as $\|v\|_{\ell^\infty}=\max_i|v_i|$ and $\|v\|_{\ell^1}=\sum_i|v_i|$.

We start by reporting two results of \citet{Wendland_2004} of fundamental relevance for the argument.
\begin{theorem}[{\citealp[Thm.~3.4]{Wendland_2004}}]\label{thm:W34}
Let $V$ be a finite-dimensional normed linear space and $\mathcal Z=\{\zeta_1,\dots,\zeta_{\mathsf M}\}\subseteq V^{*}$ a \emph{norming set} for $V$, i.e.\ the sampling operator $T:v\mapsto(\zeta_a(v))_a\in\R^{\mathsf M}$ is injective. Equip $\R^{\mathsf M}$ with $\|\cdot\|_{\infty}$, so the dual norm is $\|\cdot\|_{1}$, and set $\|T^{-1}\|=\sup_{v\ne0}\|v\|_V/\|Tv\|_{\infty}$. Then for every $\psi\in V^{*}$ there is $u\in\R^{\mathsf M}$ with $\psi(v)=\sum_a u_a \zeta_a(v)$ for all $v\in V$ and $\|u\|_{1}\le\|\psi\|_{V^{*}}\|T^{-1}\|$.
\end{theorem}
\begin{theorem}[{\citealp[Cor.~3.11]{Wendland_2004}}]\label{thm:W311}
Put $c_p:=\dfrac{3}{4(2+\sqrt3)p^{2}}$. If $Y\subseteq \mathsf B=\mathsf B(x_0,\delta)\subset\R^{d}$ satisfies $h_{Y,\mathsf B}\le c_p\delta$, then $\{\delta_y:y\in Y\}$ is a norming set for $\Pip|_{\mathsf B}$ with norming constant $\|T^{-1}\|\le2$.
\end{theorem}

The next lemma asserts the existence of stencils for the matrix $L^{\varrho_N}$ and guarantees exactness properties over intrinsic polynomials, as well as a bound over the stencil norm.

\begin{lemma}\label{lem:stencil}
Let $p\ge l$ and put
\[
   c_p:=\frac{3}{4(2+\sqrt3)p^{2}},
   \qquad
   \delta_1:=\frac{c_p}{3\,C_\inj},
\]
with $C_\inj$ as in \eqref{eq:chartdistortion}. If
\[
   \varrho_N\le\min\{1,\inj/2\}
   \qquad\text{and}\qquad
   h_N/\varrho_N\le\delta_1 ,
\]
then at every node $x_i$ there \emph{exist} weights $(c^{\varrho_N}_{ij})_{j\in S^{\varrho_N}_i}$, where $S^{\varrho_N}_i:=X_N\cap B(x_i,\varrho_N)$, such that
\begin{align}
   \sum_{j\in S^{\varrho_N}_i}c^{\varrho_N}_{ij}\,q(x_j)&=(\Lop q)(x_i)
      \qquad\text{for all }q\in\Pip,\label{eq:exactness}\\
   \|c^{\varrho_N}_i\|_{1}&\le C_1\varrho_N^{-l},
      \qquad C_1:=2\,\gamma_{p,d,l}\,A_\infty,\label{eq:normbound}
\end{align}
with $\gamma_{p,d,l}$ depending only on $p,d,l$. In particular $C_1$ depends neither on $N$ nor on $|S^{\varrho_N}_i|$.
\end{lemma}
\begin{proof}
Given $\varrho_N\le\inj/2$, the chart $\xi_i$ is defined on $B(x_i,\inj)\supseteq B(x_i,\varrho_N)$. Put
\[
   \mathsf B:= \mathsf B(0,\varrho_N)\subset\R^{d},
   \qquad
   Y:=\bigl\{y_j:=\xi_i(x_j)\ :\ x_j\in S^{\varrho_N}_i\bigr\},
   \qquad
   V:=\Pip\big|_{\mathsf B}.
\]
Define $\mathcal Z:=\{\delta_{y_j}:x_j\in S^{\varrho_N}_i\}$ with sampling operator $T$, where $\R^{|Y|}$ is equipped with $\|\cdot\|_{\infty}$. Note $Y\subseteq \mathsf B$, since $|\xi_i(x_j)|=d_g(x_i,x_j)<\varrho_N$, and $Y\ne\emptyset$ because $x_i\in S^{\varrho_N}_i$.
Given $p\ge1$ and $C_\inj\ge1$,
\[
   \delta_1\ \le\ \frac{c_p}{3}\ =\ \frac{1}{4(2+\sqrt3)p^{2}}\ <\ \frac14 .
\]
The standing hypothesis $h_N\le\delta_1\varrho_N$ therefore gives $h_N<\varrho_N/4$ and, in particular,
\begin{equation}\label{eq:rhominus}
   \varrho_N-2h_N\ >\ \varrho_N-\frac{\varrho_N}{2}\ =\ \frac{\varrho_N}{2}\ >\ 0 ,
\end{equation}
which is needed for the retraction to be well defined.

\medskip\noindent\emph{Step 1: $\mathcal Z$ is a norming set with constant $2$.}
It is clear then that the objective is to verify the hypothesis $h_{Y,\mathsf B}\le c_p\varrho_N$ of Theorem~\ref{thm:W311}. Fix $z\in\mathsf B$ and distinguish two cases. If $|z|\le\varrho_N-2h_N$, let $y:=\exp_{x_i}(z)$, so $d_g(x_i,y)=|z|$ by \eqref{eq:radial}. By definition of the fill distance there is a node $x_j$ with $d_g(y,x_j)\le h_N$, and it lies inside the ball,
\[
   d_g(x_i,x_j)\le d_g(x_i,y)+d_g(y,x_j)\le(\varrho_N-2h_N)+h_N=\varrho_N-h_N<\varrho_N ,
\]
so $x_j\in S^{\varrho_N}_i$ and $\xi_i(x_j)\in Y$. By \eqref{eq:chartdistortion} and $\xi_i(y)=z$,
\begin{equation}\label{eq:case1}
   \dist(z,Y)\le|\xi_i(x_j)-\xi_i(y)|\le C_\inj\,d_g(x_j,y)\le C_\inj h_N .
\end{equation}
If instead $\varrho_N-2h_N<|z|<\varrho_N$, retract $z$ radially to $z':=\frac{\varrho_N-2h_N}{|z|}z$, which satisfies $|z'|=\varrho_N-2h_N$, so the previous case applies to $z'$. Moreover
\[
   |z-z'|=\Bigl|z\Bigl(1-\tfrac{\varrho_N-2h_N}{|z|}\Bigr)\Bigr|
   =|z|-(\varrho_N-2h_N)<\varrho_N-(\varrho_N-2h_N)=2h_N ,
\]
using $|z|>\varrho_N-2h_N$ and $|z|<\varrho_N$ for the second inequality. Hence, by the triangle inequality,
\begin{equation}\label{eq:case2}
   \dist(z,Y)\le|z-z'|+\dist(z',Y)\le 2h_N+C_\inj h_N=(2+C_\inj)h_N .
\end{equation}
Since $C_\inj\ge1$, we have $2+C_\inj\le3C_\inj$, and if we now combine the two cases we obtain $\dist(z,Y)\le3C_\inj h_N$. Taking the supremum over $z\in \mathsf B$,
\[
   h_{Y,\mathsf B}\ \le\ 3C_\inj h_N\ \le\ 3C_\inj\,\delta_1\varrho_N
   \ =\ 3C_\inj\frac{c_p}{3C_\inj}\,\varrho_N\ =\ c_p\varrho_N .
\]
Theorem~\ref{thm:W311} now applies with $\delta=\varrho_N$.

\medskip\noindent\emph{Step 2: the functional has norm $O(\varrho_N^{-l})$.}
Define $\psi\in V^{*}$ by $\psi(q):=(\Lop q)(x_i)$, which in our convention writes $\psi(q)=\sum_{|\beta|\le l}a_\beta(x_i)\partial^{\beta}q(0)$. The dilation $q\mapsto\tilde q$, $\tilde q(z):=q(\varrho_N z)$, is an isometry of $(\Pip|_{\mathsf B(0,\varrho_N)},\|\cdot\|_{L^{\infty}})$ onto $(\Pip|_{\mathsf B(0,1)},\|\cdot\|_{L^{\infty}})$, and by the chain rule $\partial^{\beta}q(0)=\varrho_N^{-|\beta|}\partial^{\beta}\tilde q(0)$. Let $\gamma_\beta$ be the norm of the functional $\tilde q\mapsto\partial^{\beta}\tilde q(0)$, a constant depending only on $p,d,\beta$ which is finite since linear functionals are bounded on finite dimensional spaces. Then
\[
   |\partial^{\beta}q(0)|=\varrho_N^{-|\beta|}|\partial^{\beta}\tilde q(0)|
   \le\gamma_\beta\,\varrho_N^{-|\beta|}\|q\|_{L^{\infty}(\mathsf B)} ,
\]
and since $\varrho_N\le1$ gives $\varrho_N^{-|\beta|}\le\varrho_N^{-l}$ for $|\beta|\le l$,
\[
   \|\psi\|_{V^{*}}\le A_\infty\Bigl(\sum_{|\beta|\le l}\gamma_\beta\Bigr)\varrho_N^{-l}
   =\gamma_{p,d,l}A_\infty\varrho_N^{-l} .
\]
\medskip\noindent\emph{Step 3: conclusion.}
Theorem~\ref{thm:W34}, applied to $\mathcal Z$ and $\psi$, gives $u\in\R^{|Y|}$ with $\psi(q)=\sum_{j}u_j\delta_{y_j}(q)$ for all $q\in V$ and $\|u\|_{1}\le\|\psi\|_{V^{*}}\|T^{-1}\|$. Setting $c^{\varrho_N}_{ij}:=u_j$ and reading the identity back through the chart gives \eqref{eq:exactness}, while Steps 1--2 give $\|c^{\varrho_N}_i\|_1\le2\gamma_{p,d,l}A_\infty\varrho_N^{-l}$.
\end{proof}

We record here that the results of this appendix depend on the stencil coefficients only through their support in $S^{\varrho_N}_i$, through~\eqref{eq:exactness} and through~\eqref{eq:normbound}. They therefore hold for the coefficients computed in Section~\ref{sec:implementation}, namely the ones of minimal $\ell^1$ norm among those supported in $S^{\varrho_N}_i$ that satisfy~\eqref{eq:exactness}, which satisfy~\eqref{eq:normbound} with the same constant $C_1$, since the coefficients given by Lemma~\ref{lem:stencil} are feasible.

The following result provides a proof that $L^{\varrho_N}$ approximates the operator as in the standard case, but with convergence speed controlled by $\varrho_N \geq h_N/\delta_1$.
\begin{lemma}\label{lem:consistency}
Under the hypotheses of Lemma~\ref{lem:stencil}, for every $w\in C^{p+1}(\Mfd)$,
\[
   \bigl\|L^{\varrho_N}\Rest w-\Rest\Lop w\bigr\|_{\ell^{\infty}}
   \ \le\ C_2\,\varrho_N^{\,p+1-l}\,\|w\|_{C^{p+1}(\Mfd)},
\]
where $C_2:=C_1C_{\mathrm{Tay}}$ and $C_{\mathrm{Tay}}:=\frac{d^{\,p+1}}{(p+1)!}$.
\end{lemma}
\begin{proof}
Fix $i$ and work in the chart $\xi_i$, writing $f:=w\circ\exp_{x_i}$. Note that every stencil node has $|\xi_i(x_j)|=d_g(x_i,x_j)<\varrho_N$, so all evaluations below take place in $\mathsf B(0,\varrho_N)$. Let $T_i=\sum_{|\gamma|\le p}\frac{\partial^{\gamma}f(0)}{\gamma!}\,z^{\gamma}\in\Pip$ be the degree-$p$ Taylor polynomial of $f$ at $0$ and let the remainder be $r_i:=f-T_i$. Define now $g(t):=f(tz)$, $t\in[0,1]$, and expanding $g^{(m)}(t)=(z\cdot\nabla)^{m}f(tz)=\sum_{|\gamma|=m}\frac{m!}{\gamma!}z^{\gamma}\partial^{\gamma}f(tz)$, we notice that these are the numerators of the sum defining $T_i$ at $t=0$. The one-dimensional Taylor formula with integral remainder, applied to $g$ at $t=0$ and evaluated at $t=1$, therefore reads
\[
   f(z)=g(1)=\sum_{m=0}^{p}\frac{g^{(m)}(0)}{m!}
        +\frac{1}{p!}\int_{0}^{1}(1-t)^{p}g^{(p+1)}(t)\,dt
        =T_i(z)+\frac{1}{p!}\int_{0}^{1}(1-t)^{p}g^{(p+1)}(t)\,dt .
\]
Substituting the expansion of $g^{(p+1)}$ and using $\frac{(p+1)!}{p!}=p+1$,
\begin{equation}\label{eq:remainder}
   r_i(z)=(p+1)\sum_{|\gamma|=p+1}\frac{z^{\gamma}}{\gamma!}
          \int_{0}^{1}(1-t)^{p}\,\partial^{\gamma}f(tz)\,dt .
\end{equation}
Bounding $|z^{\gamma}|\le|z|^{p+1}\le\varrho_N^{p+1}$, $\int_0^1(1-t)^{p}dt=\frac{1}{p+1}$, and $|\partial^{\gamma}f(tz)|\le\|w\|_{C^{p+1}(\Mfd)}$ for $t\in[0,1]$. This gives, for $|z|\le\varrho_N$,
\[
   |r_i(z)|\le\varrho_N^{p+1}\|w\|_{C^{p+1}(\Mfd)}\sum_{|\gamma|=p+1}\frac{1}{\gamma!}
   =\frac{d^{\,p+1}}{(p+1)!}\varrho_N^{p+1}\|w\|_{C^{p+1}(\Mfd)},
\]
using $\sum_{|\gamma|=n}\tfrac{n!}{\gamma!}=d^{\,n}$. Moreover $\partial^{\gamma}_zr_i(0)=0$ for $|\gamma|\le p$, since $\partial^{\gamma}z^{\delta}|_{z=0}=\gamma!\,\delta_{\gamma\delta}$ shows that $T_i$ reproduces every derivative of $f$ at $0$ up to that order, and since $p\ge l$ this gives $(\Lop r_i)(x_i)=\sum_{|\beta|\le l}a_\beta(x_i)\partial^{\beta}r_i(0)=0$. Exactness \eqref{eq:exactness} reproduces $(\Lop T_i)(x_i)$, so that the two contributions of $T_i$ cancel and
\[
  \bigl(L^{\varrho_N}\Rest w\bigr)_i-(\Lop w)(x_i)
  =\underbrace{\textstyle\sum_jc^{\varrho_N}_{ij}T_i(x_j)-(\Lop T_i)(x_i)}_{=0}
   +\textstyle\sum_jc^{\varrho_N}_{ij}r_i(x_j)-\underbrace{(\Lop r_i)(x_i)}_{=0},
\]
thus the error is at most $\|c^{\varrho_N}_i\|_1\max_j|r_i(x_j)|\le C_1\varrho_N^{-l}C_{\mathrm{Tay}}\varrho_N^{p+1}\|w\|_{C^{p+1}}$.
\end{proof}

We are now ready to focus on the object of interest, which is the Gram matrix error. For convenience, split
\begin{equation}\label{eq:split}
   \Khat_{ij}-K_{ij}
   =\underbrace{\sum_{a,b}c^{\varrho_N}_{ia}c^{\varrho_N}_{jb}
       \bigl[\khat_{ab}-k(x_a,x_b)\bigr]}_{(\mathrm{A})_{ij}}
   +\underbrace{\sum_{a,b}c^{\varrho_N}_{ia}c^{\varrho_N}_{jb}\,k(x_a,x_b)-K_{ij}}_{(\mathrm{B})_{ij}} .
\end{equation}
First, we focus on term $\mathrm{B}$. It is clear that for this term we are interested in consistency properties similar to those of Lemma~\ref{lem:consistency}, but for a kernel. This is exactly the content of the following result.

\begin{lemma}\label{lem:double}
Under the hypotheses of Lemmas~\ref{lem:stencil}--\ref{lem:consistency} and Assumption~\ref{ass:prior},
\[
   \max_{i,j}\bigl|(\mathrm{B})_{ij}\bigr|
   \ \le\ C_3\,\varrho_N^{\,p+1-l}\,\|k\|_{C^{p+1,p+1}},
   \qquad C_3:=C_2\bigl(2C_{\Lop}+C_2\bigr).
\]
\end{lemma}
\begin{proof}
Define the function
\[
  G_j(x):=\sum_b c^{\varrho_N}_{jb}\,k(x,x_b),
\]
in which the stencil has been applied in the second entry. It is a finite linear combination of the smooth functions $k(\cdot,x_b)$ with constant coefficients, hence $G_j\in C^{\infty}(\Mfd)$, and by construction
\begin{equation}\label{eq:Gidentity}
   \bigl(L^{\varrho_N}\Rest G_j\bigr)_i=\sum_a c^{\varrho_N}_{ia}G_j(x_a)
   =\sum_{a,b}c^{\varrho_N}_{ia}c^{\varrho_N}_{jb}k(x_a,x_b).
\end{equation}
\medskip\noindent\emph{Step 1: $\|G_j\|_{C^{p+1}}$ is bounded uniformly in $j,N$.}
Fix $x\in\Mfd$, $|z_x|\le\inj/2$ and $|\gamma|\le p+1$. Differentiating under the finite sum, $\partial^{\gamma}_{z_x}\bigl(G_j\circ\exp_x\bigr)(z_x)=\sum_b c^{\varrho_N}_{jb}\,\partial^{\gamma}_{z_x}\bigl[k(\exp_x(z_x),x_b)\bigr]$. Set now $\Psi_{x,z_x,\gamma}(y):=\partial^{\gamma}_{z_x}\bigl[k(\exp_x(z_x),y)\bigr]$, then
\[
   \partial^{\gamma}_{z_x}\bigl(G_j\circ\exp_x\bigr)(z_x)
   =\sum_b c^{\varrho_N}_{jb}\Psi_{x,z_x,\gamma}(x_b)
   =\bigl(L^{\varrho_N}\Rest\Psi_{x,z_x,\gamma}\bigr)_j .
\]
Unfolding \eqref{eq:mixednorm}, it is clear that
\begin{equation}\label{eq:psibound}
   \|\Psi_{x,z_x,\gamma}\|_{C^{p+1}}\le\|k\|_{C^{p+1,p+1}}
\end{equation}
since the latter ranges over a wider set. Lemma~\ref{lem:consistency} applied to $\Psi_{x,z_x,\gamma}$, read at component $j$, then gives
\[
   \bigl|\partial^{\gamma}_{z_x}\bigl(G_j\circ\exp_x\bigr)(z_x)\bigr|
   \le\bigl|(\Lop\Psi_{x,z_x,\gamma})(x_j)\bigr|
     +C_2\varrho_N^{\,p+1-l}\|\Psi_{x,z_x,\gamma}\|_{C^{p+1}} .
\]
For the first term, $|(\Lop\Psi_{x,z_x,\gamma})(x_j)|\le A_\infty\binom{l+d}{d}\|\Psi_{x,z_x,\gamma}\|_{C^{p+1}}=C_{\Lop}\|\Psi_{x,z_x,\gamma}\|_{C^{p+1}}$. For the second, $\varrho_N\le1$ and $p+1-l\ge1$ give $\varrho_N^{\,p+1-l}\le1$. With \eqref{eq:psibound}, and taking the supremum over $x$, over $|z_x|\le\inj/2$ and over $|\gamma|\le p+1$, we obtain
\begin{equation}\label{eq:Gbound}
   \|G_j\|_{C^{p+1}(\Mfd)}\ \le \bigl(C_{\Lop}+C_2\bigr)\|k\|_{C^{p+1,p+1}} =:C_G.
\end{equation}

\medskip\noindent\emph{Step 2: outer application.}
Apply Lemma~\ref{lem:consistency} to $G_j$ and rearrange terms so that we can state
\begin{equation}\label{eq:E1}
   \bigl(L^{\varrho_N}\Rest G_j\bigr)_i=(\Lop G_j)(x_i)+E_1,
\end{equation}
where $E_1$ is the approximation error, which is bounded by $|E_1|\le C_2C_G\varrho_N^{\,p+1-l}$.

\medskip\noindent\emph{Step 3: exchange.}
By linearity of the operator and finiteness of the sum, $(\Lop G_j)(x_i)=\sum_b c^{\varrho_N}_{jb}H_i(x_b)$ with $H_i(y):=(\Lop_xk)(x_i,y)$. Since $|\beta|\le l\le p+1$, each $\partial^{\beta}_x\partial^{\delta}_yk$ appearing in $\partial^{\delta}_yH_i$ is controlled by the kernel norm, so
\begin{equation}\label{eq:Hbound}
   \|H_i\|_{C^{p+1}(\Mfd)}\le C_{\Lop}\|k\|_{C^{p+1,p+1}},
\end{equation}
uniformly in $i$ and $N$.

\medskip\noindent\emph{Step 4: inner application.}
Similarly to Step 2, apply Lemma~\ref{lem:consistency} to $H_i$ at component $j$,
\begin{equation}\label{eq:E2}
   \sum_b c^{\varrho_N}_{jb}H_i(x_b)=(\Lop H_i)(x_j)+E_2,
\end{equation}
where the error is bounded by $|E_2|\le C_2C_{\Lop}\|k\|_{C^{p+1,p+1}}\varrho_N^{\,p+1-l}$ and $(\Lop H_i)(x_j)=(\Lop_y\Lop_xk)(x_i,x_j)=K_{ij}$.

\medskip\noindent\emph{Step 5: conclusion.}
Combining \eqref{eq:Gidentity}, \eqref{eq:E1}, Step 3 and \eqref{eq:E2} gives $(\mathrm{B})_{ij}=E_1+E_2$, and \eqref{eq:Gbound}--\eqref{eq:E2} give the bound, uniformly in $i,j$.
\end{proof}

It remains to bound the term $\mathrm A$.

\begin{lemma}[Kernel-error term]\label{lem:A}
$\displaystyle\max_{i,j}\bigl|(\mathrm{A})_{ij}\bigr|
 \le\|c^{\varrho_N}_i\|_1\|c^{\varrho_N}_j\|_1\,\eta_N\le C_1^{2}\varrho_N^{-2l}\eta_N$.
Likewise, for a single application,
$\max_{i,a}\bigl|\sum_b c^{\varrho_N}_{ab}[\khat_{ib}-k(x_i,x_b)]\bigr|\le C_1\varrho_N^{-l}\eta_N$.
\end{lemma}
\begin{proof}
H\"older's inequality gives
\[
|\sum_{a,b}c^{\varrho_N}_{ia}c^{\varrho_N}_{jb}(\khat_{ab}-k(x_a,x_b))|\le(\max_{a,b}|\khat_{ab}-k(x_a,x_b)|)
\|c^{\varrho_N}_i\|_1\|c^{\varrho_N}_j\|_1.\]
Applying the bound in Lemma~\ref{lem:stencil}, the first estimate holds. The second estimate follows by the same argument related to a single application.
\end{proof}

\begin{proof}[Proof of Theorem~\ref{thm:main}]
By Lemmas~\ref{lem:double} and \ref{lem:A},
\[
   \max_{i,j}\bigl|\Khat_{ij}-K_{ij}\bigr|
   \le C_3\|k\|_{C^{p+1,p+1}}\varrho_N^{\,p+1-l}+C_1^{2}\varrho_N^{-2l}\eta_N .
\]
Define $C=C_3\|k\|_{C^{p+1,p+1}}+C_1^{2}$, which is finite by Assumption~\ref{ass:prior}. Since none of the constants depends on $N$, this gives
\[
    \max_{1\le i,j\le N}\bigl|\Khat_{ij}-K_{ij}\bigr|
    \ \le\ C\left(\varrho_N^{\,p+1-l}+\frac{\eta_N}{\varrho_N^{2l}}\right).
\]

For the last claim, set the right hand side of~\eqref{pdem:eq:gramrate} as $C\,f(\varrho_N)$ with $f(\varrho_N):=\varrho_N^{p+1-l}+\eta_N\varrho_N^{-2l}$, and note that $p+1-l \ge 1$ and $2l\ge 1$. Then
\[
    f'(\varrho_N) =\varrho_N^{-2l-1}\bigl((p+1-l)\varrho_N^{p+1+l}-2l\,\eta_N\bigr),
\]
and, since $\varrho_N$ is positive, the bracket is strictly increasing in $\varrho_N$ with a single zero. Setting $f'(\varrho_N)=0$, we find its unique minimizer is
\[
    \varrho_N^{\ast}=\Bigl(\frac{2l\,\eta_N}{p+1-l}\Bigr)^{\frac{1}{p+1+l}} .
\]
The minimiser is admissible, since $\varrho_N^{\ast}\asymp\eta_N^{1/(p+1+l)}\to 0$ and
$\eta_N(\varrho_N^{\ast})^{-2l}\asymp\eta_N^{(p+1-l)/(p+1+l)}\to 0$.
\end{proof}

\section{Error bounds for the posterior mean}\label{app:contraction}

\noindent Before starting with the discussion to prove the results regarding posterior contraction, we report here some useful notation already introduced in the main text. Define $v(x) \in \R^{N_c}$ with entries
\[ v_j(x):= \Lop_yk(x,y)|_{y=z_j},\]
and similarly $A\in\R^{N_c\times N_c}$ with entries
\[
   A_{ij}:=\Lop_x\Lop_yk(x,y)|_{x=z_i, y=z_j}.
\]
Henceforth, we denote with $\tilde \cdot$ all quantities where we adopted the desired approximation. For example, in the case of the kernel we write $\tilde k=k$ when the manifold is known and the spectral expansion of $\Delta_\Mfd$ is available, and $\tilde k=\khat$ in the implicit case. The perturbation size is defined as
\begin{equation*}
   \epsilon_N:=\max\Bigl\{\ \max_{x\in X_N}\bigl|\tilde k(x,x)-k(x,x)\bigr|,\
   \max_{x\in X_N}\max_j\bigl|\tilde v_j(x)-v_j(x)\bigr|,\
   \max_{i,j}\bigl|\tilde A_{ij}-A_{ij}\bigr|\ \Bigr\}.
\end{equation*}
For $u\in\R^{N_c}$, we introduce the quadratic form
\[
   Q_x(u):=k(x,x)-2u^{\top}v(x)+u^{\top}Au.
\]

\begin{lemma}\label{lem:epsilon}
For $\epsilon_N$ as in the previous definition, we have
\begin{enumerate}
\item \emph{(known kernel, $\tilde k=k$)} with the stencil radius taken at its smallest admissible value $\varrho_N= h_N/\delta_1$,
      $\ \epsilon_N\le C\,h_N^{\,p+1-l}\,\|k\|_{C^{p+1,p+1}}$;
\item \emph{(graph kernel, $\tilde k=\khat$)} with $\varrho_N$ as in Theorem~\ref{thm:main},
      $\epsilon_N\le C\,\eta_N^{\frac{p+1-l}{p+1+l}}$.
\end{enumerate}
\end{lemma}
\begin{proof}
For the known kernel case, the first entry vanishes while the third one follows from Lemma~\ref{lem:double} with $\varrho_N=h_N/\delta_1$. Finally, the second term follows from Lemma~\ref{lem:consistency} applied to $k(x, \cdot)$, whose $C^{p+1}$ norm is bounded by $\|k\|_{C^{p+1,p+1}}$ uniformly in $x$, again with $\varrho_N=h_N/\delta_1$. In the case of implicit kernel, the first term is at most $\eta_N$ by Assumption~\ref{ass:graph} and the third term is the object of Theorem~\ref{thm:main}. For the second term, we split into
\[
     \bigl|L^{\varrho_N}\khat(x,z_j)-\Lop_y k(x, z_j)\bigr| \leq \bigl|L^{\varrho_N}\khat(x,z_j)-L^{\varrho_N} k(x, z_j)\bigr| + \bigl|L^{\varrho_N}k(x,z_j)-\Lop_y k(x, z_j)\bigr| .
\]
The second part follows from the previous case, while the first follows from the single-application statement of Lemma~\ref{lem:A}, which gives $C_1\varrho_N^{-l}\eta_N=C_1\eta_N^{\frac{p+1}{p+1+l}}$; since $\eta_N\le1$ and $\frac{p+1}{p+1+l}>\frac{p+1-l}{p+1+l}$, this is dominated by $\varrho_N^{\,p+1-l}=\eta_N^{\frac{p+1-l}{p+1+l}}$.
\end{proof}
\begin{lemma}\label{lem:onefunctional}
For every node $x\in X_N$,
\[
   \tilde\mu(x)-u_0(x)=\tilde R_x\,u_0,
\]
where we defined $\tilde R_x:=\sum_{j=1}^{N_c}\bigl(\tilde A^{-1}\tilde v(x)\bigr)_j\,(\delta_{z_j} \circ \Lop)-\delta_x\ \in\ H_k^{*}$. Moreover, $\ \bigl\|\tilde R_x\bigr\|^{2}_{H_k^{*}}=Q_x\bigl(\tilde A^{-1}\tilde v(x)\bigr)$.
\end{lemma}
\begin{proof}
We can write data points as $y_j=(\delta_{z_j} \circ \Lop) u_0$, so $\tilde\mu(x)=[\tilde A^{-1} \tilde v(x)]^{\top}\bm y=\sum_j\bigl(\tilde A^{-1} \tilde v(x)\bigr)_j(\delta_{z_j} \circ \Lop)u_0$ and subtracting $u_0(x)=\delta_xu_0$ gives the identity. The Riesz representer of $\delta_{z_j} \circ \Lop$ in $H_k$ is $R_j=(\delta_{z_j} \circ \Lop_y)k(\cdot,y)$ and that of $\delta_x$ is $k(\cdot,x)$, so $\tilde R_x$ is represented by $\sum_j \bigl(\tilde A^{-1} \tilde v(x)\bigr)_jR_j-k(\cdot,x)$. Noting now that $\langle R_i,R_j\rangle_{H_k}=\Lop_x\Lop_yk(x,y)|_{x=z_i, y=z_j}=A_{ij}$, $\langle R_j,k(\cdot,x)\rangle_{H_k}=R_j(x)=\Lop_yk(x,y)|_{y=z_j}=v_j(x)$ and $\langle k(\cdot,x),k(\cdot,x)\rangle_{H_k}=k(x,x)$ gives
\begin{align*}
    \langle \tilde R_x,\tilde R_x\rangle_{H_k}
    &= \Bigl\langle \sum_i \bigl(\tilde A^{-1} \tilde v(x)\bigr)_iR_i-k(\cdot,x), \ \sum_j \bigl(\tilde A^{-1} \tilde v(x)\bigr)_jR_j-k(\cdot,x) \Bigr\rangle_{H_k} \\
    &= \Bigl\langle \sum_i \bigl(\tilde A^{-1} \tilde v(x)\bigr)_iR_i, \sum_j \bigl(\tilde A^{-1} \tilde v(x)\bigr)_jR_j \Bigr\rangle_{H_k}
       - 2 \Bigl\langle \sum_i \bigl(\tilde A^{-1} \tilde v(x)\bigr)_iR_i, k(\cdot, x)\Bigr\rangle_{H_k} \\
    &\qquad + \langle k(\cdot, x), k(\cdot, x)\rangle_{H_k} \\
    &=Q_x\bigl(\tilde A^{-1} \tilde v(x)\bigr).
\end{align*}
\end{proof}

We notice that, as a byproduct of this proof carried out with a general $u$ in place of $\tilde A^{-1}\tilde v(x)$, we get that $Q_x(u)$ is the squared dual norm
\[
   Q_x(u)=\Bigl\|\delta_x-\sum_{j=1}^{N_c}u_j\,(\delta_{z_j}\circ\Lop)\Bigr\|^{2}_{H_k^{*}} .
\]
In particular $Q_x\ge0$ and it is easy to see it admits a unique minimizer $A^{-1}v(x)$, with minimum value
\[
   \min_{u\in\R^{N_c}}Q_x(u)=k(x,x)-v(x)^{\top}A^{-1}v(x)=\sigma(x)^{2}.
\]
The same argument stays true for $\tilde Q_x$, $\tilde A^{-1}\tilde v(x)$ and $\tilde\sigma(x)^{2}$.

\begin{lemma}\label{lem:sandwich}
For every node $x\in X_N$,
\[
   Q_x\bigl(\tilde A^{-1} \tilde v(x)\bigr)\ \le\ \sigma(x)^{2}
     +2\,\epsilon_N\bigl(1+\Lambda_{\mathrm{Leb}}(x)\bigr)^{2}.
\]
\end{lemma}
\begin{proof}
First notice that, for any $u\in\R^{N_c}$,
\begin{align}\label{eq:QQtilde}
   \bigl|Q_x(u)-\tilde Q_x(u)\bigr|
   &\le\bigl|k(x,x)-\tilde k(x,x)\bigr|+2\|u\|_{1}\|v(x)-\tilde v(x)\|_{\infty}
     +\|u\|_{1}^{2}\max_{i,j}\bigl|A_{ij}-\tilde A_{ij}\bigr| \nonumber
   \\ &\le\epsilon_N\bigl(1+\|u\|_{1}\bigr)^{2}.
\end{align}
Both $Q_x$ and $\tilde Q_x$ are minimized over the same space $\R^{N_c}$, at $A^{-1}v(x)$ and $\tilde A^{-1}\tilde v(x)$ respectively, with minima $\sigma(x)^{2}$ and $\tilde\sigma(x)^{2}$. We now get
\begin{align*}
   Q_x\bigl(\tilde A^{-1} \tilde v(x)\bigr) &\le\ \tilde Q_x\bigl(\tilde A^{-1} \tilde v(x)\bigr)+\epsilon_N\bigl(1+\|\tilde A^{-1}\tilde v(x)\|_{1}\bigr)^{2} \\
    &\le\ \tilde Q_x\bigl(A^{-1}v(x)\bigr)+\epsilon_N\bigl(1+\|\tilde A^{-1}\tilde v(x)\|_{1}\bigr)^{2} \\
    &\le\ Q_x\bigl(A^{-1}v(x)\bigr)+\epsilon_N\bigl(1+\|\tilde A^{-1}\tilde v(x)\|_{1}\bigr)^{2}+\epsilon_N\bigl(1+\|A^{-1}v(x)\|_{1}\bigr)^{2} \\
    &=\ \sigma(x)^{2}+\epsilon_N\bigl(1+\|\tilde A^{-1}\tilde v(x)\|_{1}\bigr)^{2}+\epsilon_N\bigl(1+\|A^{-1}v(x)\|_{1}\bigr)^{2},
\end{align*}
where the second inequality follows since $\tilde A^{-1}\tilde v(x)$ minimizes $\tilde Q_x$. Both $\|A^{-1}v(x)\|_{1}$ and $\|\tilde A^{-1}\tilde v(x)\|_{1}$ are at most $\Lambda_{\mathrm{Leb}}(x)$ by the definition of $\Lambda_{\mathrm{Leb}}$, which gives the claim.
\end{proof}

\begin{proof}[Proof of Theorem~\ref{thm:contraction}]
Fix a node $x\in X_N$. By Lemma~\ref{lem:onefunctional},
\[
   \tilde\mu(x)-u_0(x)=\tilde R_x\,u_0,
   \qquad \tilde R_x\in H_k^{*},
\]
and since $u_0\in H_k$ by hypothesis,
\[
   \bigl|\tilde\mu(x)-u_0(x)\bigr| \le\ \bigl\|\tilde R_x\bigr\|_{H_k^{*}} \, \|u_0\|_{H_k} .
\]
The same lemma also gives $\bigl\|\tilde R_x\bigr\|^{2}_{H_k^{*}}=Q_x\bigl(\tilde A^{-1}\tilde v(x)\bigr)$, and Lemma~\ref{lem:sandwich} bounds this by $\sigma(x)^{2}+2\epsilon_N\bigl(1+\Lambda_{\mathrm{Leb}}(x)\bigr)^{2}$. Taking the square root of this estimate and combining with the previous one, the bound becomes
\[
   \bigl|\tilde\mu(x)-u_0(x)\bigr|
   \ \le\ \Bigl(\sigma(x)^{2}+2\epsilon_N\bigl(1+\Lambda_{\mathrm{Leb}}(x)\bigr)^{2}\Bigr)^{1/2}
     \|u_0\|_{H_k} .
\]
\end{proof}

This final part is devoted to proving the rate at which the posterior variance reduces as we increase the number of collocation points, which is used to obtain the formula~\eqref{eq:contraction_comparable}.

\begin{lemma}\label{lem:sigma}
Let $\Mfd\subset\R^{n}$ be a smooth closed $d$-dimensional manifold, let $\Lop$ be a differential operator of order $l$ with smooth coefficients, and let $\mathfrak s-l>d/2+1$, where $\mathfrak s$ is the exponent of Assumption~\ref{ass:prior}. Assume that there exists $C_{\mathrm{st}}>0$ such that
\[
   \|f\|_{L^{\infty}(\Mfd)}\le C_{\mathrm{st}}\|\Lop f\|_{L^{\infty}(\Mfd)}\qquad\text{for all }f\in H_k(\Mfd).
\]
Then $\sigma(x)=O\bigl(h_{N_c}^{\,\mathfrak s-l-d/2}\bigr)$ uniformly in $x\in\Mfd$, where $h_{N_c}$ is the geodesic fill distance of the collocation set $Z_{N_c}$.
\end{lemma}
\begin{proof}
Fix $x\in\Mfd$ and let $R_j$ be the Riesz representer of $\delta_{z_j}\circ\Lop$ in $H_k$, as in the proof of Lemma~\ref{lem:onefunctional}. Define
\[
   r_x:=k(\cdot,x)-\sum_{j=1}^{N_c}\bigl(A^{-1}v(x)\bigr)_j\,R_j\ \in H_k(\Mfd).
\]
Using $\langle R_i,R_j\rangle_{H_k}=A_{ij}$, $\langle R_j,k(\cdot,x)\rangle_{H_k}=v_j(x)$ and the symmetry of $A$, we get $(\Lop r_x)(z_i)=\langle r_x,R_i\rangle_{H_k}=v_i(x)-\bigl(AA^{-1}v(x)\bigr)_i=0$ for every $i$, and $r_x(x)=\langle r_x,k(\cdot,x)\rangle_{H_k}=k(x,x)-v(x)^{\top}A^{-1}v(x)=\sigma(x)^{2}$.
Moreover, since $r_x$ is orthogonal to every $R_j$, $\|r_x\|_{H_k}^{2}=\langle r_x,k(\cdot,x)\rangle_{H_k}=\sigma(x)^{2}$. It therefore suffices to show that every $f\in H_k(\Mfd)$ with $(\Lop f)(z_j)=0$ for all $j$ satisfies
\begin{equation}\label{eq:sigmagoal}
   \|f\|_{L^{\infty}(\Mfd)}\le C\,h_{N_c}^{\,\mathfrak s-l-d/2}\,\|f\|_{H_k},
\end{equation}
with $C$ independent of $f$, $x$ and $Z_{N_c}$. Indeed, applying~\eqref{eq:sigmagoal} to $f=r_x$ gives $\sigma(x)^{2}=r_x(x)\le C\,h_{N_c}^{\,\mathfrak s-l-d/2}\,\sigma(x)$, which is the claim, the case $\sigma(x)=0$ being trivial.

By \citet[Lemma~17.20]{Wendland_2004} and compactness, $\Mfd$ is covered by finitely many charts $(U_\ell,\varphi_\ell)$, $\ell=1,\dots,J$, such that each $\varphi_\ell$ extends to a homeomorphism from $\overline{U_\ell}$ onto the closed ball $\overline{\mathsf B}(o_\ell,b_\ell)\subset\R^{d}$, with $\varphi_\ell(U_\ell)=\mathsf B(o_\ell,b_\ell)$, and $m\,\|\varphi_\ell(y)-\varphi_\ell(v)\|_2\le d_g(y,v)$ for all $y,v\in\overline{U_\ell}$, with $m>0$ independent of $\ell$.
Fix $\ell$, assume $h_{N_c}<m\,b_\ell/2$, and let $w\in\varphi_\ell(U_\ell)$. Retract $w$ radially to the point $w'$ of the segment from $o_\ell$ to $w$ with $|w'-o_\ell|=\max\{0,|w-o_\ell|-2h_{N_c}/m\}$, so that $|w-w'|\le2h_{N_c}/m$ and $\mathsf B(w',2h_{N_c}/m)\subset\varphi_\ell(U_\ell)$, and put $y':=\varphi_\ell^{-1}(w')$. Every point at geodesic distance at most $h_{N_c}$ from $y'$ lies in $U_\ell$, since otherwise a minimizing geodesic from $y'$ would cross $\partial U_\ell$ at a point $v_0$ with $d_g(y',v_0)\le h_{N_c}$, hence $\|\varphi_\ell(v_0)-w'\|_2\le h_{N_c}/m$, while $\varphi_\ell(v_0)$ lies on the boundary sphere of $\mathsf B(o_\ell,b_\ell)$, at distance at least $2h_{N_c}/m$ from $w'$.
By definition of the fill distance there is $z_j\in Z_{N_c}$ with $d_g(y',z_j)\le h_{N_c}$, hence $z_j\in U_\ell$ and $\|\varphi_\ell(z_j)-w'\|_2\le h_{N_c}/m$. Therefore $\|\varphi_\ell(z_j)-w\|_2\le3h_{N_c}/m$, that is, the fill distance of $\varphi_\ell(Z_{N_c}\cap U_\ell)$ in $\varphi_\ell(U_\ell)$ is at most $3h_{N_c}/m$.

Let $f\in H_k(\Mfd)$ with $(\Lop f)(z_j)=0$ for all $j$, and set $g:=\Lop f$. By Assumption~\ref{ass:prior}, $H_k(\Mfd)\cong H^{\mathfrak s}(\Mfd)$, and since $\Lop$ has order $l$ and smooth coefficients, $g\in H^{\mathfrak s-l}(\Mfd)\subset C(\Mfd)$ with $\|g\|_{H^{\mathfrak s-l}(\Mfd)}\le C\|f\|_{H_k}$. For each $\ell$, the function $g\circ\varphi_\ell^{-1}\in H^{\mathfrak s-l}(\varphi_\ell(U_\ell))$ vanishes on $\varphi_\ell(Z_{N_c}\cap U_\ell)$, whose fill distance in the ball $\varphi_\ell(U_\ell)$ is at most $3h_{N_c}/m$. Writing $\mathfrak s-l=k'+s'$ with $k'\in\N$ and $s'\in(0,1]$, we have $k'\ge\mathfrak s-l-1>d/2$, so that \citet[Theorem~2.12]{narcowich2005sobolev}, in their notation with $p=2$, $q=\infty$ and $\alpha=0$, applies on the ball $\varphi_\ell(U_\ell)$ and gives, for $h_{N_c}$ small enough,
\[
   \|g\|_{L^{\infty}(U_\ell)}\le C\,h_{N_c}^{\,\mathfrak s-l-d/2}\,\|g\circ\varphi_\ell^{-1}\|_{H^{\mathfrak s-l}(\varphi_\ell(U_\ell))}\le C\,h_{N_c}^{\,\mathfrak s-l-d/2}\,\|f\|_{H_k}.
\]
Taking the maximum over the finitely many charts gives $\|\Lop f\|_{L^{\infty}(\Mfd)}\le C\,h_{N_c}^{\,\mathfrak s-l-d/2}\|f\|_{H_k}$, and the stability estimate gives~\eqref{eq:sigmagoal} with the constant multiplied by $C_{\mathrm{st}}$.
\end{proof}

\end{document}